\documentclass{jmlr} % test grayscale version
\usepackage{longtable}% for long tables

\usepackage{booktabs}
\jmlrvolume{vol}
\jmlryear{2010}

\usepackage[lined,boxed,commentsnumbered]{algorithm2e}
\usepackage{mathrsfs}
\usepackage{amsfonts}
\usepackage{mathbbold}
\usepackage{color}
\usepackage{bm}

\newtheorem{claim}{Claim}
\newcommand{\R}{\mathcal{R}}
\renewcommand{\P}{\mathcal{P}}
\newcommand{\B}{\mathcal{B}}
\newcommand{\M}{\mathcal{M}}
\newcommand{\D}{\mathcal{D}}
\newcommand{\ER}{\widehat{\R}}
\newcommand{\x}{\bm{x}}
\newcommand{\bxi}{\boldsymbol{\xi}}
\newcommand{\z}{\bm{z}}

\newcommand{\uu}{\bm{u}}
\newcommand{\w}{\bm{w}}

\newcommand{\nmlz}{\frac{1}{n - c_n}\sum_{t=c_n}^{n - 1}}
\newcommand{\nmlzi}[1]{\frac{#1}{n - 1 - t}\sum_{i=t}^{n - 2}}
\newcommand{\nmlzt}{\frac{1}{n - 1 - t}\sum_{i=t}^{n - 2}}

\renewcommand{\Pr}{\mathbb{P}}
\newcommand{\nmlzb}{\frac{1}{|\B_{t-1}|}\sum_{j\in\B_{t-1}}}
\newcommand{\A}{\mathbf{A}}
\newcommand{\bb}{\mathbf{B}}
\newcommand{\X}{\mathbf{X}}

\def\bbbe{{\rm I\!E}} % for expectation
\def\bbbr{{\rm I\!R}} % for real field
\def\bbbn{{\rm I\!N}} % for integer field
\def\bbbi{{\rm I\!I}} % for the indicator function

\title[Online Learning with Pairwise Loss Functions]{Online Learning with Pairwise Loss Functions}

  \author{\Name{Yuyang Wang} \Email{ywang02@cs.tufts.edu}
  \\
  \Name{Roni Khardon} \Email{roni@cs.tufts.edu}\\
  \addr
Department of Computer Science, Tufts University, Medford, MA 02155, USA
  \AND
  \Name{Dmitry Pechyony} \Email{dpechyon@akamai.com}
  \\
  \Name{Rosie Jones} \Email{rejones@akamai.com}\\
  \addr Akamai Technologies, 8 Cambridge Center, Cambridge, MA 02142, USA
  }

\editor{Editor's name}
\begin{document}

\maketitle

\begin{abstract}
Efficient online learning with pairwise loss functions is a crucial component in building large-scale learning system that maximizes the area under the Receiver Operator Characteristic (ROC) curve. In this paper we investigate the generalization performance of online learning algorithms with pairwise loss functions. We show that the existing proof techniques for generalization bounds of online algorithms with a univariate loss can not be directly applied to pairwise losses. In this paper, we derive the first result providing data-dependent bounds for the average risk of the sequence of hypotheses generated by an \emph{arbitrary} online learner in terms of an easily computable statistic, and show how to extract a low risk
hypothesis from the sequence. We demonstrate the generality of our results by applying it to two important problems in machine learning. First, we analyze two online algorithms for bipartite ranking; one being a natural extension of the perceptron algorithm and the other using online convex optimization. Secondly, we provide an analysis for the risk bound for an online algorithm for supervised metric learning.
\end{abstract}

\section{Introduction}
\label{sec:intro}
The standard framework in learning theory considers learning from examples
$Z^n= \{\left(\x_t, y_t\right)$ $\in\mathcal{X}\times\mathcal{Y}\},\
t=1,2,\cdots, n$, (independently) drawn at random from an unknown probability distribution $\D$
on $\mathcal{Z}:=\mathcal{X}\times\mathcal{Y}$ (e.g. $\mathcal{X}=\bbbr^d$ and $\mathcal{Y} = \bbbr$).
Typically a \emph{univariate} loss function $\ell(h, (\x, y))$ is adopted to measure the performance of the hypothesis $h:\mathcal{X}\rightarrow\mathcal{Y}$, for example, $\ell(h, (\x,y)) = (h(\x) - y)^2$ for regression or $\ell(h,\x,y) = \bbbi_{[h(\x)\neq y]}$ for classification.  The aim of learning is to find a hypothesis that generalizes well, i.e. has small expected risk $\bbbe_{(\x, y)}\ell(h, (\x, y))$.

In this paper we study learning in the context of \emph{pairwise loss functions},
that
depend on pairs of examples
and can be expressed as
$\ell\left(h, (\x, y), (\uu, v)\right)$ where the hypothesis is applied to pairs of examples, i.e. $h:\mathcal{X}\times\mathcal{X}\rightarrow \bbbr$.
Pairwise loss functions capture ranking
problems that are important for a wide range of applications. For example, in the \emph{supervised ranking} problem one wishes to learn a ranking function that predicts the correct ordering of objects. The hypothesis $h$ is called a ranking rule such that $h(\x,\uu) > 0$ if $\x$ is ranked higher than $\uu$ and vice versa. The \emph{misranking loss}~\citep{clemencon2008ranking,peel2010empirical} is a pairwise loss such that
\[
    \ell_{\text{rank}}\left(h, (\x, y), (\uu, v)\right) = \bbbi_{\left[(y-v)(h(\x,\uu)) < 0\right]},
\]
where $\bbbi$ is the indicator function and
the loss is 1 when the examples are ranked in the wrong order. The goal of learning is to find a hypothesis $h$ that minimizes the expected misranking risk $\R(h)$,
\begin{equation}
\mathcal{R}(h): = \bbbe_{(\x, y)}\bbbe_{(\uu,
v)}\left[\ell\left(h, (\x, y), (\uu, v)\right)\right].
\end{equation}
In many interesting cases, finding a ranking rule amounts to learning a good scoring function $s:\mathcal{X}\rightarrow\bbbr$ such that $h(\x, \uu)  = s(\x) - s(\uu)$. Therefore, higher ranked examples will have higher scores. Another application comes from \emph{distance metric learning}, where the learner wishes to learn a distance metric such that examples that share the same label should be close while ones from different labels are far away from each others.

%This problem, especially the bipartite ranking problem where $\mathcal{Y} =
%\{+1, -1\}$, has been extensively studied over the past decade in the {\em
%  batch setting}, i.e.,\ where the entire sequence $Z^n$ is presented to the
%learner in advance of learning. On the empirical end, many algorithms have
%been proposed and successfully applied, for example, AUC
%Support Vector Machine (SVM)~\citep{brefeld2005auc}, Ranking SVM~\citep{joachims2002optimizing},
%and RankBoost~\citep{freund2003efficient}. Several theoretical studies also investigated the batch setting, deriving risk bounds for specific algorithms
%\citep{freund2003efficient,rudin2005margin}, and uniform convergence bounds for empirical estimates of the risk
%\citep{agarwal2005generalization,agarwal2005stability,agarwal2009generalization} and related $U$-statistics
%\citep{clemencon2008ranking, peel2010empirical,rejchel2012ranking}. We provide detailed discussion of related work is in Section~\ref{sec:discuss}.

This problem, especially the bipartite ranking problem where $\mathcal{Y} =
\{+1, -1\}$, has been extensively studied over the past decade in the {\em
  batch setting}, i.e.,\ where the entire sequence $Z^n$ is presented to the
learner in advance of learning. \cite{freund2003efficient}
gave generalization bounds for the RankBoost algorithm, based on
the uniform convergence results for classification.
\cite{agarwal2005generalization} derived uniform convergence bounds for the
bipartite ranking loss, using a quantity called rank-shatter coefficient, which
generalizes ideas from the classification setting.
Agarwal et.\ al.\
%\cite{agarwal2005stability}
provided bounds for the bipartite ranking
problem \citep{agarwal2005stability} and the general ranking problem
\citep{agarwal2009generalization} using ideas from
algorithmic stability.
\cite{rudin2005margin} approached a closely related problem where the goal is
to correctly rank only the top of the ranked list, and derived generalization
bounds based on $L_\infty$ covering number.
Recently, several authors investigated
oracle inequalities for pairwise-based quantities via the formalization of $U$-statistics
\citep{clemencon2008ranking, rejchel2012ranking} using empirical processes. \cite{peel2010empirical} gave an empirical Bernstein inequality for higher order $U$-statistics. Another thread comes from the perspective of reducing ranking problems to the more familiar classification problems~\citep{kotllowski2011bipartite, ertekin2011equivalence, agarwal2012surrogate}.

In this paper we investigate the generalization performance
of {\em online learning} algorithms,
where examples are presented in sequence, in the context of pairwise loss
functions. Specifically, on each round $t$, an online learner receives an instance $\x_t$ and predicts a label $\hat{y}_t$ according to the current hypothesis $h_{t-1}$. The true label $y_t$ is revealed and $h_{t-1}$ is updated. The goal of the online learner is to minimize the expected risk w.r.t. a pairwise loss function $\ell$.

Over the past two decades, online learning algorithms have been studied extensively,
and theoretical
results provide relative loss bounds, where the online learner competes
against the best hypothesis (with hindsight) on the same
sequence. Conversions of online learning algorithms and their performance
guarantees to provide generalization performance in the batch setting
have also been investigated (e.g.,\
\citep{KearnsLiPiVa87b,littlestone1990mistake, FreundSc99,Zhang05}).
\cite{cesa2004generalization} provided a general online-to-batch conversion result that holds under
some mild assumptions on the loss function. Given a univariate loss function $\ell$, a sample $Z^n$ and an ensemble of hypotheses $\{h_1, h_2, \cdots, h_n\}$ generated by an online learner $\mathcal{A}$, the following cumulative loss of $\mathcal{A}$ is defined as
\[
M_n = M_n(Z^n) = \frac{1}{n}\sum_{t=1}^n\ell(h_{t-1}, Z_t).
\]
\cite{cesa2008improved} showed (as a refined version of the bound in \citep{cesa2004generalization}) that one can extract a hypothesis $\widehat{h}$ from the ensemble such that
\[
\Pr\left(\R(\widehat{h}) \geqslant M_n + \mathcal{O}\left(\frac{\ln^2n}{n} + \sqrt{M_n\frac{\ln n}{n}}\right)\right) \leqslant \delta.
\]
Therefore, if one can develop an online learning algorithm with bounded cumulative loss for every possible realization of $Z^n$, then its generalization performance is guaranteed. A sharper bound exists when the loss function is strongly convex \citep{kakade2009generalization}. The key step of these derivations is to realize that $V_{t-1} = \R(h_{t-1}) - \ell(h_{t-1}, Z_t)$ is a martingale difference sequence. Thus one can use martingale concentration inequalities (Azuma's inequality or Friedman Inequality) to bound $\sum V_t$. Unfortunately,
this property no longer holds for pairwise loss functions.

%some extensions are reported in
%\citep{cesa2008improved,kakade2009generalization}.
%The main tool in this work is the use of martingale concentration
%inequalities (the Hoeffding-Azuma Inequality and the Friedman Inequality) to derive bounds on the average risk of the sequence of
%hypotheses generated by the learning algorithm in terms of a data-dependent statistic. Essentially, this relies on the fact that the differences $V_t$ of the empirical loss $\ell(h_{t-1}, \z_{t})$ and the true risk $\R(h_{t-1}):=\bbbe_{\z}[\ell(h_{t-1}, \z)]$ form a martingale sequence.

Of course, as mentioned for example in the work of
~\citet[Sec.~4.2]{peel2010empirical}, one can slightly adapt an existing online learning classification algorithm
(e.g.,\ perceptron), feeding it with data sequence $\breve{\z}_{t} := \left(\z_{2t-1}, \z_{2t}\right)$
and modifying the update function accordingly.
In this case, previous analysis
\citep{cesa2008improved} does apply. However, this does not make full use of the
examples in the training sequence. In addition, empirical results show that
this naive algorithm, which corresponds to the
algorithm for
online maximization of the area under the ROC curve (AUC) with a
buffer size of one in~\citep{zhao2011online}, is inferior to algorithms that
retain some form of the history of the sequence. Alternatively, it is tempting to consider feeding the online algorithm with pairs $\breve{\z}^t_i = (\z_i, \z_t), i < t$ on each round. However, in this case, existing results would again fail because $\breve{\z}^t_i$ are not i.i.d. Hence, a natural question
is whether we can prove data dependent generalization bounds
based on the online pairwise loss.

This paper provides a positive answer to this question for
a large family of pairwise loss functions. On each round $t$, we measure $M_t$, the average loss of $h_{t-1}$ on examples $(\z_i, \z_t), i < t$. Let $\M^n$ denote
the average loss, averaging $M_t$ over $t \geqslant (1-c)n$ on a
training sequence of length $n$ where $c$ is a small constant. The main
result of this paper, provides a model selection mechanism to select one of
the hypotheses of an \emph{arbitrary} online learner, and states that the probability that the risk of the chosen
hypothesis $\widehat{h}$ satisfies,
\[
 \R(\widehat{h}) \geqslant \M^n + \epsilon
\]
is at most
\[
2\left[\mathcal{N}\left(\mathcal{H},
\frac{\epsilon}{32Lip(\phi)}\right) + 1\right]
\exp\left\{-\frac{(cn-1)\epsilon^2}{256} + 2\ln n\right\}.
\]
Here $\mathcal{N}(\mathcal{H}, \eta)$ is the $L_\infty$ covering number for
the hypothesis class
$\mathcal{H}$ and $Lip(\phi)$ is determined by the Lipschitz constant
of the loss function (definitions and details are provided in the following
sections). Thus, our results provide an online-to-batch conversion for pairwise loss functions. We demonstrate our results with the following two applications:
\begin{enumerate}
  \item We analyze two online learning algorithms for the bipartite ranking problem. We first provide an analysis of a natural generalization
of the perceptron algorithm to work with pairwise loss functions,
that provides loss bounds in both the separable case and the inseparable case.
As a
byproduct, we also derive a new simple proof of the best $L_1$  based mistake
bound for the perceptron algorithm in
the inseparable case.
Combining with our main results we provide the first \emph{online} algorithm with
corresponding risk bound for bipartite ranking. Secondly, we analyze another algorithm using the online convex optimization techniques, with similar risk bounds.
  \item Several online metric
learning algorithms have been proposed with corresponding regret analyzes, but the generalization performance of these algorithms has been left open, possibly because no tools existed to provide online-to-batch conversion with pairwise loss functions. We provide risk bounds for an online algorithm for distance metric learning combining with the results for online convex optimization with matrix argument.
\end{enumerate}

The rest of this paper is organized as follows. Section~\ref{sec:main} defines the problem and states our main technical theorem and Section~\ref{sec:proof} provides a sketch of the proof. We provide model selection results and risk analysis for convex and general loss functions in Section~\ref{sec:mselect}. In Section~\ref{sec:app}, we describe our online algorithm for bipartite ranking and analyze it. The results in sections~\ref{sec:main}-\ref{sec:app} are given for a model and algorithms with an ``infinite buffer'', that is, where the update of the online learner at step $t$ depends on the entire history of the sequence, $\z_1,\cdots,\z_{t-1}$. Section~\ref{sec:finite} shows that the results and algorithms can be adopted to a buffer of limited size. Interestingly, to guarantee convergence our results require that the buffer size grows logarithmically with the sequence size. Section~\ref{sec:metric} is devoted to the analysis of online metric learning. Finally, we conclude the paper and discuss possible future directions in Section~\ref{sec:conclu}.

\section{Main Technical Result}
\label{sec:main}
Given a sample $Z^n = \{\z_1, \cdots, \z_n\}$ where $\z_i$=$(\x_i, y_i)$ and a sequence of hypotheses $h_0, h_1, \cdots$, $h_n$ generated by an online learning algorithm, we define the sample statistic $M^n$ as
\begin{equation}
\label{eqn:mn}
\begin{split}
\M^n(Z^n) = \nmlz M_t(Z^t),\qquad M_t(Z^t) =
\frac{1}{t-1}\sum_{i=1}^{t-1} \ell\left(h_{t-1}, \z_t, \z_i\right),
\end{split}
\end{equation}
where $c_n = \lceil c\cdot n \rceil$ and $c\in(0,1)$ is a small
positive constant. $M_t(Z^t)$ measures the performance of the hypothesis $h_{t-1}$ on the next example $\z_t$ when paired with all previous examples.
Note that instead of considering all the $n$ generated hypotheses, we only consider the average of the hypotheses $h_{c_n - 1},\cdots, h_{n-2}$ where the statistic $M_t$ is reliable and the last two hypotheses $h_{n-1}, h_{n}$ are discarded for technical reasons. In the following, to simplify the notation, $\M^n$ denotes $\M^n(Z^n)$ and $M_t$ denotes $M_t(Z^t)$. We define $f\wedge g \equiv \min(f,g)$ and $f\vee g \equiv \max(f,g)$.

As in~\citep{cesa2004generalization}, our goal is to
bound the average risk of the sequence of hypotheses in terms of $\M^n$,
which can be obtained using the following theorem.
\begin{theorem}
\label{thm:concetra}
Assume the hypothesis space $\left(\mathcal{H},
\|\cdot\|_\infty\right)$ is compact. Let $h_0, h_1, \cdots,
h_{n}\in\mathcal{H}$ be the ensemble of hypotheses generated by an
arbitrary online algorithm working with a pairwise loss function
$\ell$ such that,
  \[\ell(h, \z_1, \z_2) = \phi(y_1 - y_2, h(\x_1,\x_2)), \]
  where $\phi:\mathcal{Y}\times\mathcal{Y}\rightarrow [0, 1]$ is a Lipschitz function w.r.t. the second variable with a finite Lipschitz constant
  $\text{Lip}(\phi)$. Then, $\forall c >0, \forall \epsilon > 0$, we have for sufficiently large $n$
\begin{equation}
\Pr\left\{\nmlz\mathcal{R}(h_{t-1}) \geqslant \M^n +
\epsilon\right\} \leqslant \left[2\mathcal{N}\left(\mathcal{H},
\frac{\epsilon}{16\text{Lip}(\phi)}\right) + 1\right]
\exp\left\{-\frac{(cn-1)\epsilon^2}{64} + \ln n\right\}.
\end{equation}

\end{theorem}
Here the $L_\infty$ covering number $\mathcal{N}(\cal{H}, \eta)$ is defined to be the minimal $\ell$ in $\bbbn$ such that there exist $\ell$ disks in $\mathcal{H}$
with radius $\eta$ that cover $\cal{H}$. We make the following remarks.
\begin{remark}
  Let $\bbbe_t$ denote $\bbbe_{\z_t}[\cdot|\z_1, \cdots,\z_{t-1}]$. It can be seen that $\bbbe_t[M_t] - \R(h_{t-1})$ is no longer a martingale difference sequence. Therefore, martingale concentration inequalities that are usually used in online-to-batch conversion do not directly yield the desired bound.
\end{remark}
\begin{remark}
   We need the assumption that the hypothesis space $\mathcal{H}$ is compact so that its covering number $\mathcal{N}(\cal{H}, \eta)$ is
   finite.
As an example, suppose $\mathcal{X}\subset \bbbr^d$ and the hypothesis space is the class of linear functions that lie within a ball
   $B_R(\bbbr^d) = \{\w\in\bbbr^d:\underset{\x\in\mathcal{X}}{\sup}\langle \w, \x\rangle \leqslant R\}. $
   It can be shown \citep[see][chap.~5]{cucker2007learning} that the covering number is one if $\eta > R$ and otherwise
   \begin{equation}
   \label{eqn:cover}
     \mathcal{N}(B_R, \eta) \leqslant \left(\frac{2R}{\eta} + 1\right)^d.
   \end{equation}
\end{remark}
\begin{remark}
\label{rem:loss}
  We say that $f(s, t)$ is Lipschitz w.r.t~the second argument if $\forall s, |f(s,t_1) - f(s, t_2)|\leqslant Lip(f)\|t_1 - t_2\|$. This form of the pairwise loss function is not restrictive and is widely used. For example, in the supervised ranking problem, we can take the hinge loss as
  \[\ell_\text{hinge}(h, \z_1, \z_2) = \phi(y_1 - y_2, h(\x_1) - h(\x_2)) = \left[1 - (h(\x_1) - h(\x_2))(y_1 - y_2)\right]_+,\]
  which can be thought as a surrogate function for $\ell_\text{rank}$.  Since $\phi$ is not bounded, we define the bounded hinge loss using $\widetilde{\phi}(s, t) = \min([1-st]_+, 1)\in[0,1]$ if $s\neq 0$ and 0 otherwise. We next show that $\widetilde{\phi}$ is Lipschitz. This is trivial for $y=0$. For $y\neq 0$, when the first argument is bounded by a constant $C$, $\widetilde{\phi}(y, \cdot)$ satisfies
  \[ \big|\widetilde{\phi}(y, x_1) - \widetilde{\phi}(y, x_2)\big|  \leqslant \big|\left[1 - yx_1\right]_+ -
  \left[1 - yx_2\right]_+\big| \leqslant \|yx_1 - yx_2\| \leqslant
  C\|x_1 - x_2\|.
  \]
   Alternatively, one can take the square loss, i.e.
  $
   \ell(h, \z_1, \z_2)$ = $\left[1 - (h(\x_1) - h(\x_2))(y_1 - y_2)\right]^2.
  $
  If its support is bounded then $\ell$ is Lipschitz.
\end{remark}

\section{Proof of the Main Technical Result}
\label{sec:proof}
The proof is inspired by the work of~\citep{cucker2002mathematical, agarwal2005generalization,rudin2009p}. The proof makes use of the Hoeffding-Azuma inequality, McDiarmid's inequality, symmetrization techniques and covering numbers of compact spaces.

\begin{proof}[Proof of Theorem~\ref{thm:concetra}]
By the definition of $\M^n$ (see (\ref{eqn:mn})), we wish to bound
\begin{equation}
\Pr_{Z^n\sim\D^n}\left(\nmlz\mathcal{R}(h_{t-1}) - \nmlz M_t \geqslant \epsilon\right),
\end{equation}
which can be rewritten as
\begin{align}
\label{eqn:probsplit}
&\Pr\left(\nmlz\bigg[\mathcal{R}(h_{t-1}) - \bbbe_t[M_t]\bigg] +  \nmlz\bigg[\bbbe_t[M_t] - M_t\bigg] \geqslant \epsilon\right) \nonumber\\
&\quad \leqslant
\Pr\left(\nmlz\bigg[\mathcal{R}(h_{t-1}) -
\bbbe_t[M_t]\bigg] \geqslant
\frac{\epsilon}{2}\right) + \Pr\left(
\nmlz\bigg[\bbbe_t[M_t] - M_t\bigg]
\geqslant \frac{\epsilon}{2} \right).
\end{align}
Thus, we can bound the two terms separately. The proof consists of four parts, as follows.

\subsection*{Step 1: Bounding the Martingale difference}
First consider the second term in (\ref{eqn:probsplit}). We have that $V_t =
(\bbbe_t[M_t] - M_t)/(n-c_n)$ is a martingale
difference sequence, i.e. $\bbbe_t[V_t] = 0$. Since the loss function is bounded in $[0,1]$, we have
$
|V_t| \leqslant {1}/(n - c_n),\ t = 1,\cdots, n.
$
Therefore by the Hoeffding-Azuma inequality, $\sum_t V_t$
can be bounded such that
\begin{equation}
\label{eqn:second} \Pr_{Z^n\sim\D^n}\left( \nmlz\bigg[\bbbe_t[M_t] - M_t\bigg]
\geqslant \frac{\epsilon}{2} \right) \leqslant
\exp\left\{-\frac{(1-c)n\epsilon^2}{2}\right\}.
\end{equation}
\subsection*{Step 2: Symmetrization by a ghost sample $\Xi^n$}
In this step we bound the first term in (\ref{eqn:probsplit}). Let us start by introducing a ghost sample
$\Xi^n=\{\bxi_j\}=\{(\tilde{\x}_j, \tilde{y}_j)\}, j=1,\cdots, n$
where each $\bxi_j$ follows the same distribution as $\z_j$. Recall the definition of
$M_t$ and define $\widetilde{M}_t$ as
\begin{equation}
\label{eqn:tm} {M}_t = \frac{1}{t-1}\sum_{j = 1}^{t-1} \ell(h_{t-1},
\z_t, \z_j), \qquad \widetilde{M}_t = \frac{1}{t-1}\sum_{j =
1}^{t-1} \ell(h_{t-1}, \z_t, \bxi_j).
\end{equation}
The difference between $\widetilde{M}_t$ and $M_t$ is that $M_t$ is
the sum of the loss incurred by $h_{t-1}$ on the current instance $\z_t$ and all the
previous examples $\z_j, j = 1, \cdots, {t-1}$ \emph{on
which $h_{t-1}$ is trained}, while $\widetilde{M}_t$ is the loss
incurred by the same hypothesis $h_{t-1}$ on the current instance $\z_t$ and \emph{an independent set of examples $\bxi_j, j = 1, \cdots, t-1$}.

\begin{claim}
\label{clm:sym}
The following equation holds
{\small{\begin{equation}
\label{eqn:symm}
\Pr_{Z^n\sim\D^n}\left(\nmlz\left[\mathcal{R}(h_{t-1}) - \bbbe_t[M_t]\right]
\geqslant \epsilon\right) \leqslant 2\Pr_{\substack{Z^n\sim\D^n \\\Xi^n\sim\D^n}}\left(\nmlz
\left[\bbbe_t[\widetilde{M}_t] - \bbbe_t[M_t]\right]\geqslant
\frac{\epsilon}{2}\right),
\end{equation}}}
whenever $n > 2/(\epsilon^2c^2)$.
\end{claim}
Notice that the probability measure on the right hand side of (\ref{eqn:symm}) is on $Z^n\times\Xi^n$.

\begin{proof}[\emph{Sketch of the proof of Claim~\ref{clm:sym}}]
It can be seen that the RHS (without the factor of 2) of (\ref{eqn:symm}) is at least

{\footnotesize{
\begin{displaymath}
\begin{split}
&\Pr_{\substack{Z^n\sim\D^n \\ \Xi^n\sim\D^n}}\left(\left\{\nmlz\left[\mathcal{R}(h_{t-1}) - \bbbe_t[M_t]\right] \geqslant \epsilon\right\} \bigcap
\left\{\bigg|\nmlz\left[\bbbe_t[\widetilde{M}_t] - \mathcal{R}(h_{t-1})\right]\bigg| \leqslant \frac{\epsilon}{2}\right\}\right)\\
&=
\bbbe_{Z^n\sim\D^n}\left[\bbbi_{\left\{\nmlz \left[\mathcal{R}(h_{t-1}) - \bbbe_t[M_t]\right] \geqslant \epsilon\right\}}\cdot\Pr_{\Xi^n\sim\D^n}\left(\bigg|
\nmlz\left[\bbbe_t[\widetilde{M}_t] - \mathcal{R}(h_{t-1})\right] \bigg|\leqslant \frac{\epsilon}{2}\bigg|Z^n\right)\right].
\end{split}
\end{displaymath}}}

Since $\bbbe_{\Xi^n\sim\D^n}\bbbe_t[\widetilde{M}_t] = \mathcal{R}(h_{t-1})$, by Chebyshev's inequality
\begin{equation}
\label{eqn:cheyb}
\begin{split}
\Pr_{\Xi^n\sim\D^n}\left( \bigg|\nmlz\bigg[\bbbe_t[\widetilde{M}_t] - \mathcal{R}(h_{t-1})\bigg]\bigg| \leqslant \frac{\epsilon}{2}\bigg|Z^n\right) \geqslant 1 - \frac{\textbf{Var}\left\{\nmlz\bbbe_t[\widetilde{M}_t]\right\}}{\epsilon^2/4}.
\end{split}
\end{equation}
To bound the variance, we first investigate the largest variation when changing one random variable $\bxi_j$ with others fixed.
From (\ref{eqn:tm}), it can be easily seen that changing any of
the $\bxi_j$ varies each $\bbbe_t[\widetilde{M}_t]$, where $t > j$ by at most by $1/(t-1)$. Recall that we are only concerned with $\bbbe_t[\widetilde{M}_t]$ when $t \geqslant c_n$.
Therefore, we can see that the variation of $\nmlz\bbbe_t[\widetilde{M}_t]$
regarding the $j$th example $\bxi_j$ is bounded by
\begin{equation}
\label{eqn:cjupper}
  c_j = \frac{1}{n-c_n}\left[\sum_{t = (j \vee c_n + 1)}^{n-1}\frac{1}{t - 1}\right]
\leqslant \frac{1}{n-c_n}\left[\sum_{t = c_n + 1}^{n-1}\frac{1}{c_n}\right]  \leqslant \frac{1}{cn}.
\end{equation}
Thus, by Theorem~9.3 in \citep{devroye1996probabilistic}, we have
\begin{equation}
\label{eqn:sumci}
\begin{split}
\textbf{Var}\left(\nmlz\bbbe_t[\widetilde{M}_t]\right) &\leqslant
\frac{1}{4}\sum_{i=1}^n c_i^2  \leqslant \frac{1}{4c^2n}.
\end{split}
\end{equation}
Thus, whenever $\epsilon^2c^2n > 2$, the LHS of (\ref{eqn:cheyb}) is greater or equal than $1/2$. This completes the proof of Claim~\ref{clm:sym}.
\end{proof}
\subsection*{Step 3: Uniform Convergence}
In this step, we show how one can bound the RHS of (\ref{eqn:symm}) using uniform convergence techniques, McDiarmid's inequality and $L_\infty$ covering number. Our task reduces to bound the following quantity
\begin{equation}
\label{eqn:final}
\Pr_{Z^n\sim\D^n, \Xi^n\sim\D^n}\left(\nmlz\left[\bbbe_t[\widetilde{M}_t] - \bbbe_t[M_t]\right]\geqslant {\epsilon}\right).
\end{equation}
Here we want to bound the probability of the large deviation between the empirical performance of the ensemble of hypotheses on the sequence $Z^n$ on which they are learned and on an independent sequence $\Xi^n$. Since $h_{t}$ relies on $\z_1, \cdots, \z_t$ and is independent of $\{\bxi_t\}$, we resort to
\emph{uniform convergence techniques} to bound this probability. Define
$
L_t(h_{t-1}) =\bbbe_t[\widetilde{M}_t] - \bbbe_t[M_t].
$
Thus we have
\begin{align}
\label{eqn:final2}
\Pr_{Z^n\sim\D^n, \Xi^n\sim\D^n}\left(\nmlz L_t(h_{t-1}) \geqslant \epsilon\right)
&\leqslant
\Pr\left(\underset{\hat{h}_{c_n},\cdots, \hat{h}_{n-1}}{\sup}\left[\nmlz L_t(\hat{h}_{t-1})\right] \geqslant \epsilon\right)\nonumber\\
&\leqslant
\sum_{t=c_n}^{n-1}\Pr_{Z^t\sim\D^t, \Xi^t\sim\D^t}\left(\underset{\hat{h}\in\mathcal{H}}{\sup}\left[L_t(\hat{h})\right]\geqslant\epsilon\right).
\end{align}
To bound the RHS of (\ref{eqn:final2}), we start with the following lemma.% whose proof is given in the appendix.
\begin{lemma}
\label{lemma:lt}
Given any function
$f\in\mathcal{H}$ and any $t\geqslant 2$
\begin{equation}
\label{eqn:Ltconc}
\Pr_{Z^t\sim\D^t, \Xi^t\sim\D^t}\left(L_t(f)\geqslant\epsilon\right) \leqslant
\exp\left\{-(t-1)\epsilon^2\right\}.
\end{equation}
\end{lemma}
The proof which is given in the appendix shows that $L_t(f)$ has a bounded variation of $1/(t-1)$ when changing each of its $2(t-1)$ variables and applies McDiarmid's inequality. Finally, our task is to bound
$
\Pr(\underset{f\in\mathcal{H}}{\sup}\left[L_t(f)\right]\geqslant\epsilon).
$
Consider the simple case where the hypothesis space $\mathcal{H}$
is finite, then using the union bound, we immediately get the desired bound.
Although $\mathcal{H}$ is not finite, a similar analysis goes through based on the assumption that
$\mathcal{H}$ is compact. We will follow \cite{cucker2002mathematical} and show how this can be bounded.
The next two lemmas (see proof of Lemma~\ref{lemma:ltv} in the appendix) are used to derive Lemma~\ref{thm:supL}.
\begin{lemma}
\label{lemma:ltv}
  For any two functions $h_1, h_2\in\mathcal{H}$, the following
  equation holds
  \[
  L_t(h_1) - L_t(h_2) \leqslant 2\text{Lip}(\phi)\|h_1 -
  h_2\|_\infty.
  \]
\end{lemma}
\begin{lemma}
\label{lemma:uniform}
  Let $\mathcal{H} = S_1\cup\cdots\cup S_\ell$ and $\epsilon > 0$. Then
  \[\Pr\left(\underset{h\in\mathcal{H}}{\sup} L_t(h) \geqslant \epsilon\right) \leqslant
  \sum_{j=1}^\ell\Pr\left(\underset{h\in S_j}{\sup} L_t(h) \geqslant \epsilon\right)\]
\end{lemma}
\begin{lemma}
\label{thm:supL}
For every $2\leqslant t\leqslant n$, we have
\begin{equation}
\label{eqn:supLt}
\Pr\left(\underset{h\in\mathcal{H}}{\sup}\left[L_t(h)\right]\geqslant\epsilon\right) \leqslant  \mathcal{N}\left(\mathcal{H},
\frac{\epsilon}{4\text{Lip}(\phi)}\right)\exp\left\{-\frac{(t-1)\epsilon^2}{4}\right\}.
\end{equation}
\end{lemma}
\begin{proof}[\emph{Proof of Lemma~\ref{thm:supL}}]
  Let $\ell = \mathcal{N}\left(\mathcal{H},\frac{\epsilon}{2\text{Lip}(\phi)}\right)$ and consider $h_1, \cdots, h_\ell$ such that the disks $D_j$ centered at $h_j$ and with radius $\frac{\epsilon}{2\text{Lip}(\phi)}$ cover $\mathcal{H}$. By Lemma~\ref{lemma:ltv}, we have
 \[
  |L_t(h) - L_t(h_j)| \leqslant 2\text{Lip}(\phi)\|h - h_j\|_\infty \leqslant \epsilon.
  \]
   Thus, we get
  \[
  \Pr\left(\underset{h\in D_j}{\sup} L_t(h) \geqslant 2\epsilon\right) \leqslant \Pr\left(L_t(h_j)\geqslant \epsilon\right)\]
  Combining this with (\ref{eqn:Ltconc}), and Lemma~\ref{lemma:uniform} and replacing $\epsilon$ by $\epsilon/2$, we have (\ref{eqn:supLt}).
\end{proof}
Combining (\ref{eqn:supLt}) and (\ref{eqn:final2}), we have
\begin{equation}
\label{eqn:ssupL}
\Pr\left(\nmlz L_t(h_{t-1}) \geqslant \epsilon\right)
\leqslant \mathcal{N}\left(\mathcal{H},
\frac{\epsilon}{4\text{Lip}(\phi)}\right)
n\exp\left\{-\frac{(c_n-1)\epsilon^2}{4}\right\}.
\end{equation}
This shows why we need to discard the first $c_n$ hypotheses in the ensemble. If we include $h_2$ for example, according to (\ref{eqn:supLt}), we have $\Pr(L_2(f)\geqslant \epsilon) \leqslant e^{-\epsilon^2}$. As $n$ grows, this heavy term remains in the sum, and the desired bound cannot be obtained.

\subsection*{Step 4: Putting it all together}
From (\ref{eqn:symm}) and (\ref{eqn:final2}) and substituting $\epsilon$ with $\epsilon/4$ in (\ref{eqn:ssupL}), we have
{\small{\begin{equation}
\label{eqn:first}
\Pr_{Z^n\sim\D^n}\left(\nmlz\left(\mathcal{R}(h_{t-1}) -
\bbbe_t[M_t]\right) \geqslant
\frac{\epsilon}{2}\right) \leqslant 2\mathcal{N}\left(\mathcal{H},
\frac{\epsilon}{16\text{Lip}(\phi)}\right)
n\exp\left\{-\frac{(c_n-1)\epsilon^2}{64}\right\}.
\end{equation}}}
From (\ref{eqn:first}) and (\ref{eqn:second}), accompanied with the fact that (\ref{eqn:first}) decays faster than (\ref{eqn:second}), we complete the proof for Theorem~\ref{thm:concetra}.
\end{proof}

\section{Model Selection}
Following \cite{cesa2004generalization} our main tool for finding a
good hypothesis from the ensemble of hypotheses generated by the online learner is to choose the one that has a small empirical risk.  We
measure the risk for $h_t$ on the remaining $n-t$ examples, and
penalize each $h_t$ based on the number of examples on which it is
evaluated, so that the resulting upper bound on the risk is
reliable. Our construction and proofs (in the appendix) closely follow
the ones in \citep{cesa2004generalization}, using large deviation
results for $U$-statistics~\citep[see][Appendix]{clemencon2008ranking}
instead of the Chernoff bound.

\label{sec:mselect}
\subsection{Risk Analysis for Convex losses}
\label{sec:convex}
If the loss function $\phi$ is convex in its second argument and $\mathcal{Y}$ is convex, then we
can use the average hypothesis $ \bar{h} = \nmlz h_{t-1}$. It is
easy to show
that $\bar{h}$ achieves the
desired bound (the proof is in the Appendix), i.e.
\begin{equation}
\label{eqn:convex}
\Pr\left(\mathcal{R}(\bar{h}) \geqslant M^n(Z^n) +
\epsilon\right) \leqslant \left[2\mathcal{N}\left(\mathcal{H},
\frac{\epsilon}{16\text{Lip}(\phi)}\right) + 1\right]
\exp\left\{-\frac{(cn-1)\epsilon^2}{64} + \ln n\right\}.
\end{equation}

\subsection{Risk Analysis for General Losses}
\label{sec:general}
Define the empirical risk of hypothesis $h_t$ on $\{\z_{t+1}, \cdots, \z_n\}$ as $\widehat{\mathcal{R}}(h_t, t+1)$
\[
\widehat{\mathcal{R}}(h_t, t+1) =\begin{pmatrix}
  n - t\\
  2
\end{pmatrix}^{-1}\sum_{k > i,\ i \geqslant t +1}^n\ell(h_t, \z_i, \z_k).
\]
The hypothesis $\widehat{h}$ is chosen to minimize the following \emph{penalized empirical risk},
\begin{equation}
\label{eqn:hh}
\widehat{h} = \underset{c_n - 1\leqslant t < n - 1}{\text{argmin}}(\widehat{\mathcal{R}}(h_t, t+1) + c_\delta(n-t)),
\end{equation}
where
\[
c_\delta(x) = \sqrt{\frac{1}{x - 1}\ln\frac{2(n - c_n)(n - c_n + 1)}{\delta}}.
\]
Notice that we discard the last two hypotheses so that any $\ER(h_t, t+1), c_n - 1 \leqslant t\leqslant n - 2$ is well defined. The following theorem, which is the main result of this paper, shows that the risk of $\widehat{h}$ is bounded relative to $\M^n$. The proof of Theorem~\ref{thm:main} is in Appendix E.
\begin{theorem}
\label{thm:main}
Let $h_0, \cdots, h_{n}$ be the ensemble of hypotheses generated by an arbitrary online algorithm $\mathcal{A}$ working with a pairwise loss $\ell$ which satisfies the conditions given in Theorem~\ref{thm:concetra}. $\forall \epsilon > 0$, if the hypothesis is chosen via (\ref{eqn:hh}) with the confidence $\delta$ chosen as
\[
\delta = 2(n-c_n+1)\exp\left\{-\frac{(n-c_n)\epsilon^2}{64}\right\},
\]
then, when $n$ is sufficiently large, we have
\[
\begin{split}
 &\Pr\left(\R(\widehat{h}) \geqslant  M^n + \epsilon\right)\leqslant 2\left[\mathcal{N}\left(\mathcal{H},
\frac{\epsilon}{32\text{Lip}(\phi)}\right) + 1\right]
\exp\left\{-\frac{(cn-1)\epsilon^2}{256} + 2\ln n\right\}.
\end{split}
\]
\end{theorem}

\section{Application: Online Algorithms for Bipartite Ranking}
\label{sec:app}

In the bipartite ranking problem we are given a sequence of labeled
examples $\z_t = (\x_t, y_t)\in \bbbr^d\times \{-1, +1\}, t =
1,\cdots, n$. Minimizing the \emph{misranking loss}
$\ell_{\text{rank}}$ under this setting is equivalent to maximizing
the AUC, which measures the probability that
$f$ ranks a randomly drawn positive example higher than a randomly
drawn negative example. This problem has been studied extensively in
the batch setting, but the corresponding online problem has not been
investigated until recently. In this section, we investigate two online algorithms, analyze their relative loss bounds and combine them with the main result to derive risk bounds for them.

%In this section, we will propose an online algorithm (OAM-I), analyze its mistake bound and finally provide the risk analysis. On the other hand, we supply another algorithm using Online Convex Optimization (OCO) and combining with our generalization bounds to give its risk analysis.

\subsection{Online AUC Max with Infinite Buffer (OAM-I)}
Recently, \cite{zhao2011online}
proposed an online algorithm using linear hypotheses for this problem
based on reservoir
sampling, and derived bounds on the expectation of the regret of
this algorithm.
\cite{zhao2011online} use the hinge loss (that
bounds the 0-1 loss) to derive the regret bound.
The hinge loss is Lipschitz, but it is not bounded and therefore not
suitable for our risk bounds.
Therefore, in the following we use a modified loss function
where we bound the Hinge loss in $[0, 1]$ such
that \[\ell(f, \z_t, \z_j) = \widetilde{\phi}((y_t - y_j)/2, f(\x_t)
   - f(\x_j))\] where $\widetilde{\phi}$ is defined in
   Remark~\ref{rem:loss}.
Using this loss function together with Theorem~\ref{thm:main} all
we need is an
online algorithm that minimizes $\M^n$ (or an upper bound of $\M^n$) and
this guarantees generalization ability of the corresponding online
learning algorithm. To this end, we propose the following
perceptron-like algorithm, shown in Algorithm~\ref{alg:oam}, and provide
loss bounds for this algorithm.
Notice that the algorithm does not treat each pair of examples
separately, and instead for each $\z_t$ it makes a large combined
update using its loss relative to all previous examples.
Our algorithm
corresponds to
the algorithm of \cite{zhao2011online} with an infinite buffer,
but it uses a different learning rate and different loss function
which are important in our proofs.
\begin{algorithm}
\label{alg:oam}
%\LinesNumbered
\textbf{Initialize}: $\w_0 = \mathbf{0}$\;
\Repeat{the last instance} { At the $t$-th iteration, receive a training instance $\z_t = (\x_t, y_t)\in \bbbr^d\times \{-1, +1\}$.

\For{$j \leftarrow 1$ \KwTo $t-1$}{
   Calculate instantaneous loss $\ell_j^t = \ell(\w_{t-1}, \z_t, \z_j)$.
}
Update the weight vector such that
\[ \begin{split}
  \w_t &= \w_{t-1} + \frac{1}{t-1}\sum_{j=1}^{t-1}\ell_j^ty_t(\x_t - \x_j)
\end{split}.
\]
}
\caption{Online AUC Maximization (OAM) with Infinite Buffer.}
\end{algorithm}

\begin{theorem}
\label{thm:insep}
  Suppose we are given an arbitrary sequence of examples $\z_t=(\x_t,y_t), t=1,\cdots, n$, and let $\uu$ be any unit vector. Assume $\underset{t}{\max}\|\x_t\| \leqslant R$ and define
  \[
  M = \sum_{t=2}^n\frac{1}{t-1}\left[\sum_{j=1}^{t-1}\ell_j^t\right],  M^* = \sum_{t=2}^n\frac{1}{t-1}\left[\sum_{j=1}^{t-1}\hat{\ell}_{j}^{t}(\uu)\right],
\]
where \[\hat{\ell}_{j}^{t}(\uu) = \bbbi_{y_t \neq y_j}\cdot\left[\gamma - \langle \uu, \frac{1}{2}(y_t - y_j)(\x_t - \x_j)\rangle\right]_+\]. That is, $M^*$ is the cumulative average hinge loss $\uu$ suffers on the sequence with margin $\gamma$. Then, after running Algorithm~\ref{alg:oam} on the sequence, we have
\[
 M \leqslant \left(\frac{\sqrt{4R^2 + 2} + \sqrt{\gamma M^*}}{\gamma} \right)^2.
 \]
When
the data is linearly separable by margin $\gamma$, (i.e. there exists an unit vector $\uu$ such that $\hat{\ell}_j^t=0,\forall t\leqslant n, j < t$), we have $M^* = 0$ and the bound is constant.
  \end{theorem}

 \begin{proof}[\emph{Proof of Theorem~\ref{thm:insep}}]
   First notice that $\w_0 = \w_1 = 0$ and we also have the following fact
   \[
     \left[\gamma - \langle \uu, \frac{1}{2}(y_t - y_j)(\x_t - \x_j)\rangle\right]_+ \geqslant \gamma - \langle \uu, \frac{1}{2}(y_t - y_j)(\x_t - \x_j) \rangle,
  \]
  which implies that when $y_t \neq y_j$,
  \begin{equation}
  \label{eqn:hingeineq}
      \langle \uu, y_t(\x_t - \x_j)\rangle = \langle \uu, \frac{1}{2}(y_t - y_j)(\x_t - \x_j) \rangle \geqslant \gamma - \hat{\ell}_{j}^{t}(\uu).
  \end{equation}
  On the other hand, when $y_t = y_j$, then ${\ell}_{j}^{t}=0$. Thus we can write

\begin{align}
\label{eqn:inseplb}
  \langle \w_t, \uu \rangle &= \langle \w_{t-1}, \uu \rangle + \frac{1}{t-1}\sum_{j=1}^{t-1}\ell^t_j\langle \uu, y_t(\x_t - \x_j)\rangle \nonumber\\
  & \geqslant \langle \w_{t-1}, \uu \rangle + \frac{1}{t-1}\sum_{j=1}^{t-1}\ell_j^t (\gamma - \hat{\ell}_{j}^{t}(\uu))  = \langle \w_{t-1}, \uu \rangle + \frac{\gamma}{t-1}\sum_{j=1}^{t-1}\ell_j^t - \frac{1}{t-1}\sum_{j=1}^{t-1}\ell_j^t\cdot\hat{\ell}_{j}^{t}(\uu) \nonumber\\
  & \geqslant \langle \w_{t-1}, \uu \rangle + \frac{\gamma}{t-1}\sum_{j=1}^{t-1}\ell_j^t - \frac{1}{t-1}\sum_{j=1}^{t-1}\hat{\ell}_{j}^{t}(\uu) \qquad (\because \ell_j^t \in [0, 1]) \nonumber\\
  \Rightarrow\quad \langle \w_t, \uu \rangle &\geqslant \sum_{t=2}^n\left[\frac{\gamma}{t-1}\sum_{j=1}^{t-1}\ell_j^t - \frac{1}{t-1}\sum_{j=1}^{t-1}\hat{\ell}_{j}^{t}(\uu)\right] = \gamma M - M^*.
\end{align}

On the other hand, we have,

\begin{align}
\label{eqn:sepub}
  \|\w_t\|^2 &= \|\w_{t-1}\|^2 + \frac{2}{t-1}\sum_{j=1}^{t-1}\ell^t_j\langle \w_{t-1}, y_t(\x_t - \x_j)\rangle + \left\|\frac{1}{t-1}\sum_{j=1}^{t-1}\ell^t_jy_t(\x_t - \x_j)\right\|^2\nonumber\\
  &\leqslant \|\w_{t-1}\|^2 +  \frac{2}{t-1}\sum_{j=1}^{t-1}\ell^t_j + 4R^2\left(\frac{1}{t-1}\right)^2\left(\sum_{j=1}^{t-1}\ell^t_j\right)\cdot \left(\sum_{j=1}^{t-1}\ell^t_j\right) \nonumber\\
  &\qquad (\because \ell_j^t > 0 \Rightarrow \langle \w_{t-1}, y_t(\x_t - \x_j)\rangle \leqslant 1)\nonumber\\
  &\leqslant \|\w_{t-1}\|^2 +  \frac{2}{t-1}\sum_{j=1}^{t-1}\ell^t_j + 4R^2\left(\frac{1}{t-1}\right)^2 \left(\sum_{j=1}^{t-1}\ell^t_j\right)\cdot (t-1) \qquad (\because \ell_j^t \in [0, 1])\nonumber\\
  &= \|\w_{t-1}\|^2 + (4R^2 + 2)\left[\frac{1}{t-1}\sum_{j=1}^{t-1}\ell^t_j\right]\nonumber\\
  \Rightarrow\quad \|\w_n\|^2 &\leqslant (4R^2 + 2)\sum_{t=2}^n\left[\frac{1}{t-1}\sum_{j=1}^{t-1}\ell_j^t\right] = (4R^2 + 2)M
\end{align}

Combining (\ref{eqn:inseplb}) and (\ref{eqn:sepub}), we have
$
(\gamma M - M^*)^2 \leqslant (4R^2 + 2)M,
$
which yields
\[
\begin{split}
M &\leqslant \frac{\gamma M^* + (2R^2 + 1) + \sqrt{(2R^2+1)(\gamma M^* + 2R^2 + 1)}}{\gamma^2} \\
&\leqslant  \frac{\gamma M^* + (4R^2 + 2) + \sqrt{(2R^2+1)\gamma M^*}}{\gamma^2} \leqslant \left(\frac{\sqrt{4R^2 + 2} + \sqrt{\gamma M^*}}{\gamma} \right)^2
\end{split}
\]
 \end{proof}
 We therefore get the risk bound for the proposed algorithm as follows.
\begin{theorem}
\label{thm:final}
  Let $\w_0, \cdots, \w_{n-1}$ be the ensemble of hypotheses generated by Algorithm~\ref{alg:oam}. $\forall \epsilon > 0$, if the hypothesis $\widehat{w}$ is chosen via (\ref{eqn:hh}) with the confidence $\delta$ chosen to be
\[
\delta = 2(n-c_n+1)\exp\left\{-\frac{(n-c_n)\epsilon^2}{64}\right\},
\]
then the probability that
\[
\R(\widehat{\w}) \geqslant  \frac{1}{n-c_n}\left[\left(\frac{\sqrt{4R^2 + 2} + \sqrt{\gamma M^*}}{\gamma} \right)^2\right] + \epsilon
\]
is at most
\[
2\left[\left(\frac{128R^2\sqrt{5n}}{\epsilon} + 1\right)^d + 1\right]
\exp\left\{-\frac{(cn-1)\epsilon^2}{256} + 2\ln n\right\}.
\]
\end{theorem}
\begin{proof}[\emph{Proof of Theorem~\ref{thm:final}}]
By (\ref{eqn:sepub}), we can easily see that $\|\w_t\| \leqslant \sqrt{n(4R^2+2)}, t=1,\cdots, n$, therefore we have
$\|\w_t\|\cdot\|\x\| \leqslant R^2\sqrt{5n},\ \forall t \leqslant n
$. Therefore, we can take the hypothesis space to be
\[\mathcal{H} = \{\w\in\bbbr^d:\underset{\|\x\|\leqslant R}{\max}\ |\langle \w, \x \rangle| \leqslant R^2\sqrt{5n}\}.\] By (\ref{eqn:cover}), the covering number can be calculated.
On the other hand, from the definition in (\ref{eqn:mn}), it is easy to see that $\M^n\leqslant M/(n-c_n)$. Finally, combining  Theorem~\ref{thm:main} and Theorem~\ref{thm:insep} concludes the proof.
\end{proof}

\subsection{Mistake Bound for Perceptron}
  Interestingly, we can apply our proof strategy in
  Theorem~\ref{thm:insep} to analyze the Perceptron algorithm in the
  inseparable case. This recovers the best known bound in terms of the one-norm of the hinge losses (given by \citep[Theorem~8]{gentile2003robustness} and \citep[Theorem~2]{shalev2005new}), but using a simple direct proof.
  \begin{theorem}{\citep{gentile2003robustness,shalev2005new}}
    Let $(\x_1, y_1), \cdots, (\x_n, y_n)$ be a sequence of examples with $\|\x_i\|\leqslant R$. Let $\uu$ be any unit vector and let $\gamma > 0$. Define the one-norm of the hinge losses as
  \[
   D_1 =\sum_{t=1}^n\ell_t,\qquad \ell_t = \left[\gamma - y_t\langle \uu, \x_t\rangle\right]_+.
  \]
  Then the number of mistakes the  perceptron algorithm makes on this sequence is bounded by
  \[
  \left(\frac{R + \sqrt{\gamma D_1}}{\gamma}\right)^2.
  \]
  \end{theorem}
  \begin{proof}
    Let $m_t = \bbbi_{\text{sgn}(\w_t\cdot\x_t) \neq y_t}$ so that the total number of mistakes is $M = \sum_t m_t$. Then, as usual, the upper bound is $\|\w_n\|^2 \leqslant R^2M$. On the other hand, using the fact that
    $
     \ell_t = \left[\gamma - y_t\langle\uu, \x_t\rangle\right]_+ \geqslant \gamma - y_t\langle\uu, \x_t\rangle,
    $
    which implies
    $
     y_t\langle\uu, \x_t\rangle \geqslant \gamma - \ell_t,
    $ we have the lower bound
    {\begin{align}
    \label{eqn:plb}
       \langle\w_{t + 1}, \uu\rangle &= \langle\w_t, \uu\rangle + y_t\langle\x_t, \uu\rangle m_t \geqslant \langle\w_t, \uu\rangle + (\gamma - \ell_t)m_t \nonumber\\
        &= \langle \w_t, \uu\rangle + \gamma m_t - \ell_tm_t \geqslant \langle\w_t, \uu\rangle + \gamma m_t - \ell_t \qquad (\because m_t \leqslant 1)\nonumber\\
        \Rightarrow\quad \langle \w_n, \uu \rangle &\geqslant \sum_{t = 1}^{n}\gamma m_t - \sum_{t = 1}^{n}\ell_t = \gamma M - D_1.
    \end{align}} Combing the upper bound $R^2M$ with (\ref{eqn:plb}), we get
    $
    (\gamma M - D_1)^2 \leqslant R^2M.
    $
    Solving the quadratic equation, we have
    \[
    \begin{split}
      M &\leqslant \frac{1}{2\gamma^2}\left[2\gamma D_1 + R^2 + \sqrt{4\gamma R^2D_1 + R^4}\right] \leqslant \frac{1}{2\gamma^2}\left[2\gamma D_1 + 2R^2 + \sqrt{4\gamma R^2D_1}\right]\\
      &=\frac{1}{\gamma^2}\left[R^2 + \gamma D_1 + R\sqrt{D_1}\right]\leqslant \left(\frac{R + \sqrt{\gamma D_1}}{\gamma}\right)^2
    \end{split}.
    \]
  \end{proof}
\subsection{Online Projected Gradient Descent for bipartite ranking}
We start by reviewing the online learning problem with univariate loss functions under the framework of online convex optimization (OCO). We are given a convex set $K$ and at each step $t$, the algorithm selects a hypothesis $\w_{t-1}\in K$. Nature then reveals a convex loss function $f_t$ and the algorithm suffers a loss $f_t(\w_{t-1})$. The goal of the online learner is to perform well comparing to the best $\w^*$ which is obtained as if the whole data sequence is observed beforehand. More formally, we wish to develop algorithms that achieve a low value of regret $R(T)$ after round $T$, which is defined as follows:
\[
R(T) = \sum_{t=1}^Tf_t(\w_{t-1}) - \inf_{\w\in K}\sum_{t=1}^Tf_t(\w).
\]
One algorithm that has performance guarantees is the \emph{Online Projected Gradient Descent} algorithm, which consists of the following three steps:
\begin{enumerate}
  \item Choose a learning rate $\eta$.
  \item Choose $\w_0$ to be an arbitrary point in the convex set $K$.
  \item For all $t=1,2,\cdots, T$,
  \[
    \w_{t+1} = \P_K(\x_t - \eta\nabla f_t(\w_{t-1})),
  \]
  where $\P_K$ is the projection operator.
\end{enumerate}
The following theorem~\citep{zinkevich} shows that this algorithm achieves $\mathcal{O}(\sqrt{T})$ regret.
\begin{theorem}
\label{thm:oco}
Assume that $K$ is bounded, closed and non-empty. Let
  \[ D = \max_{\w\in K}\|\w_0 - \w\|. \]
  Assume $f_t$ is convex and $\nabla f_t$ exists for all $t$, and define
  \[
  G = \max_{t\in[T], \w\in K}\|\nabla f_t(\w)\|
  \]
  to be the maximum $l_2$ norm of the gradient of any $f_t$ in the set $K$. Choose $\eta = D/G\sqrt{T}$, then the regret of the Online Projected Gradient Descent algorithm after time $T$ is at most:
  \[
  R(T) \leqslant GD\sqrt{T}.
  \]
\end{theorem}

Next, consider applying this algorithm to the bipartite ranking problem, where the hypothesis $\w$ is restricted to reside in a convex set $K$. Suppose $\|\x_t\| \leqslant R$ and $\max_{\w\in K}\|\w\|\leqslant U$, we define the normalized hinge loss as follows
  \begin{equation}
    \label{eqn:ocohinge}
    \ell_{hinge}(\w, \z_1,\z_2) = \frac{1}{1 + 4RU}\left[1 - \w^T(\x_1 - \x_2)(y_1-y_2)\right]_+,
  \end{equation}
  which is convex in $\w$ and bounded in $[0,1]$.  In each step, nature selects the following convex loss function
  \[
    f_t(\w_{t-1}) = \frac{1}{t-1}\sum_{j = 1}^{t-1}\ell_{hinge}(\w_{t-1}, \z_t, \z_j),
  \]
which is also a convex function because it is a convex combination of convex functions.

Thus, we can bound the sub-gradient as
  \[
  \begin{split}
    \|\nabla f_t(\w)\| &= \frac{1}{t-1}\frac{1}{1 + 4RU}\bigg\|-\sum_{j=1}^{t-1}(\x_t-\x_j)(y_t-y_j)\bigg\|\leqslant \frac{4R}{1+4RU}\leqslant \frac{1}{U}.
  \end{split}
  \]
Matching the terminology in Theorem~\ref{thm:oco}, we have $D = U$ and $G = 1/U$. It is easy to see that the Lipschitz constant of~(\ref{eqn:ocohinge}) is also upper bounded by $1/U$. Therefore, we have the algorithm for online bipartite ranking given in Algorithm~\ref{alg:oco}.
\begin{algorithm}
\label{alg:oco}
\textbf{Initialize}: $\w_0 = \mathbf{0}$ and $\eta = \frac{U^2}{\sqrt{T}}$\;
\Repeat{the last instance} { At the $t$-th iteration, receive a training instance $\z_t = (\x_t, y_t)\in \bbbr^d\times \{-1, +1\}$.

Update the weight vector such that
\[ \begin{split}
  \w_t &= P_{K}\left[\w_{t-1} - \frac{1}{1+4RU}\cdot\frac{1}{t-1}\sum_{j=1}^{t-1}\eta(y_t-y_j)(\x_t - \x_j)\cdot\bbbi_{\ell_{hinge}(\w_{t-1}, \z_t, \z_j) > 0}\right]
\end{split}.
\]
}
\caption{Online Projected Gradient Descent for Bipartite Ranking.}
\end{algorithm}

From Theorem~\ref{thm:oco}, we see that the regret of Algorithm~\ref{alg:oco} is bounded by $\sqrt{n}$. Let $M^*$ denote $\inf_w\sum f_t(\w)$, i.e., the online loss of the optimal $\w$. Using~(\ref{eqn:convex}), we have the following theorem
\begin{theorem}
\label{thm:ocog}
   Let $\w_0, \cdots, \w_{n-1}$ be the ensemble of hypotheses generated by Algorithm~\ref{alg:oco} and let
   \[
     \overline{\w} = \nmlz\w_t
   \]
then the probability that
\[
\R(\overline{\w}) \geqslant  \frac{1}{n-c_n}\left(M^* + \sqrt{n}\right) + \epsilon
\]
is at most
\[
\left[2\left(\frac{32R}{\epsilon} + 1\right)^d + 1\right]
\exp\left\{-\frac{(cn-1)\epsilon^2}{64} + \ln n\right\}.
\]
\end{theorem}
Notice that the above bound is worse than the one in Theorem~\ref{thm:final} when the data is linearly separable. In the inseparable case, when the optimal cumulative loss $M^*=\mathcal{O}(n)$, it yields a better bound.

%In principle,
%one could turn the results of \cite{zhao2011online} into a high
%probability bound on $M$ using the Chebyshev's inequality and then
%use Theorem~\ref{thm:main} to analyze its risk. However, this does
%not yield exponential convergence as above.  It would be interesting
%to investigate this further to improve the probabilistic analysis of
%the loss bound of \cite{zhao2011online}, or integrate the buffer analysis into the risk bound of this paper
%to yield tighter results.

\section{Risk Bounds for Algorithms with Finite Buffers}
\label{sec:finite}
A natural criticism is that Algorithm~\ref{alg:oam} and~\ref{alg:oco} are not real online
algorithms due to the fact that the entire sample is stored and at each
iteration $t$, the update requires $\mathcal{O}(t)$ time while online
algorithms should have $\mathcal{O}(1)$ time per step.

To make it a real online algorithm, one can constrain the number of updates at each iteration. The idea is that at the $t$-th
iteration, instead of keeping all previous $t-1$ examples, we keep buffer $\mathcal{B}_t$, whose cardinality can not exceed a predefined size $|\B|$, that has a sample of the history. We call this type of online bipartite ranking algorithm \emph{OAM with finite buffer}.

One realization of this idea is using the ``reservoir sampling'' techniques
from~\citep{zhao2011online} (Random OAM) where the buffer $\B_t$ is maintained via reservoir sampling. \cite{zhao2011online}
gave a bound on the expectation of the
cumulative loss $\mathcal{L} = \sum_t\sum_j\ell^t_j$.  Translating
their bound to our notation we get $\bbbe[M] = M^* +
\mathcal{O}(\sqrt{n})$ where the expectation is over randomly sampled
instances in the buffer. However, when the data are
linearly separable, the cumulative loss given by this bound grows as $\mathcal{O}(\sqrt{n})$ which is worse than the bound we provided. In principle,
one could turn the results of \cite{zhao2011online} into a high
probability bound on $M$ using the Chebyshev's inequality and then
use Theorem~\ref{thm:main} to analyze its risk. However, this does
not yield exponential convergence as above. Therefore, a natural question is whether we can provide similar analysis for OAM with finite buffer. The answer is positive.

In the following, we provide the complete analysis. We give the generalization bounds for online learners with the finite buffer. Let us redefine the sample statistic $\M^n_{\B}$ as
\begin{equation}
\label{eqn:mn1}
\begin{split}
\M_{\B}^n(Z^n) = \nmlz M^{\B}_t(Z^t),\qquad M^{\B}_t(Z^t) =
\nmlzb \ell\left(h_{t-1}, \z_t, \z_j\right).
\end{split}
\end{equation}
The difference is that for each hypothesis $h_{t-1}$, the performance is evaluated on $\x_t$ and the examples in the buffer $\B_t$, which is a subset of the previous examples. Throughout this section, we assume a sequential buffer update strategy, or First In First Out (FIFO), as this simplifies the analysis. We believe that the same results hold for other random buffer strategies, e.g., the reservoir sampling, but leave this for future work. At iteration $t$, if $\B_{t-1}$ already hits the maximum size $|\B|$, we substitute the oldest example in $\B_{t-1}$ with $\z_t$; otherwise $\z_t$ is added to $\B_{t-1}$. We can extend Theorem~\ref{thm:concetra} as follows. The proof is similar and is provided in Appendix \ref{sec:appen:concetra2} for completeness.
\begin{theorem}
\label{thm:concetra2}
Assume the preconditions in Theorem~\ref{thm:concetra} hold and $|\B_t| = (t-1) \wedge |\B|$ where $|\B|$ is a predefined upper bound for the buffer size. Then, $\forall c >0, \forall \epsilon > 0$, we have for sufficiently large $n$
\begin{equation}
\Pr\left\{\nmlz\mathcal{R}(h_{t-1}) \geqslant \M_{\B}^n +
\epsilon\right\} \leqslant \left[2\mathcal{N}\left(\mathcal{H},
\frac{\epsilon}{16\text{Lip}(\phi)}\right) + 1\right]
\exp\left\{-\frac{(|\B_{\lfloor cn \rfloor}|-1)\epsilon^2}{64} + \ln n\right\}.
\end{equation}
\end{theorem}

Similarly, using the same technique of extracting a single hypothesis from an ensemble as in Section~\ref{sec:general}, we have the following theorem
\begin{theorem}
\label{thm:mainf}
Assume the preconditions in Theorem~\ref{thm:concetra} hold. $\forall \epsilon > 0$, if the hypothesis is chosen via (\ref{eqn:hh}) with the confidence $\delta$ chosen as
\[
\delta = 2(n-c_n+1)\exp\left\{-\frac{(n-c_n)\epsilon^2}{64}\right\},
\]
then, when $n$ is sufficiently large, we have
\[
\begin{split}
 &\Pr\left(\R(\widehat{h}) \geqslant  M_{\B}^n + \epsilon\right)\leqslant 2\left[\mathcal{N}\left(\mathcal{H},
\frac{\epsilon}{16\text{Lip}(\phi)}\right) + 1\right]
\exp\left\{-\frac{(|\B_{cn}| -1)\epsilon^2}{256} + 2\ln n\right\}.
\end{split}
\]
\end{theorem}

Thus, we can see that the buffer size must grow faster than $\ln(n)$. We then introduce Algorithm~\ref{alg:oamf}, the finite buffer analog of Algorithm~\ref{alg:oam}.

\begin{algorithm}
\label{alg:oamf}
%\LinesNumbered
\textbf{Initialize}: $\w_0 = \mathbf{0}, \B_0 = \emptyset$\;
\Repeat{the last instance} { At the $t$-th iteration, receive a training instance $\z_t = (\x_t, y_t)\in \bbbr^d\times \{-1, +1\}$.

Update the weight vector such that
\[ \begin{split}
  \w_t &= \w_{t-1} + \frac{1}{|\B_{t-1}|}\sum_{j\in\B_{t-1}}\ell_j^ty_t(\x_t - \x_j)
\end{split}.
\]
Update the buffer using FIFO such that $\B_t = \B_{t-1} \cup \z_t$; if $|\B_t| > |\B|$, remove $\z_{t-|\B|+1}$ out of the buffer.
}
\caption{Online AUC Maximization (OAM) with Finite Buffer.}
\end{algorithm}
Algorithm~\ref{alg:oamf} is very similar to the algorithm in~\citep{zhao2011online} where the random buffer strategy is substituted with FIFO. We obtain the following mistake bound.

\begin{theorem}
\label{thm:insepf}
  Suppose we are given a sequence of examples $\z_t, t=1,\cdots, n$, and let $\uu$ be any unit vector. Assume $\underset{t}{\max}\|\x_t\| \leqslant R$ and define
  \[
  M_{\B} = \sum_{t=2}^n\frac{1}{|\B_{t-1}|}\left[\sum_{j\in\B_{t-1}}\ell_j^t(\uu)\right],  M_{\B}^* = \sum_{t=2}^n\frac{1}{|\B_{t-1}|}\left[\sum_{j\in\B_{t-1}}\hat{\ell}_{j}^{t}\right].\]
Then, after running Algorithm~\ref{alg:oamf} on the sequence, we have
\[
 M_{\B} \leqslant \left(\frac{\sqrt{4R^2 + 2} + \sqrt{\gamma M_{\B}^*}}{\gamma} \right)^2.
 \]
  \end{theorem}
The proof is almost identical to that of Theorem~\ref{thm:insep} and we include it in Appendix~\ref{sec:insepf} for completeness.

Similarly, using the model selection approach describe before, we have
\begin{theorem}
\label{thm:finalb}
  Let $\w_0, \cdots, \w_{n-1}$ be the ensemble of hypotheses generated by Algorithm~\ref{alg:oamf}. $\forall \epsilon > 0$, if the hypothesis $\widehat{\w}_{\B}$ is chosen via (\ref{eqn:hh}) with the confidence $\delta$ chosen to be
\[
\delta = 2(n-c_n+1)\exp\left\{-\frac{(n-c_n)\epsilon^2}{64}\right\},
\]
then the probability that
\[
\R(\widehat{\w}_{\B}) \geqslant  \frac{1}{n-c_n}\left[\left(\frac{\sqrt{4R^2 + 2} + \sqrt{\gamma M_{\B}^*}}{\gamma} \right)^2\right] + \epsilon
\]
is at most
\[
2\left[\left(\frac{32R^2\sqrt{5n}}{\epsilon} + 1\right)^d + 1\right]
\exp\left\{-\frac{(|\B_{cn}|-1)\epsilon^2}{256} + 2\ln n\right\}.
\]
\end{theorem}

We have similar extension for Algorithm~\ref{alg:oco}, as described in Algorithm~\ref{alg:ocof}. It has the same risk bound as in Theorem~\ref{thm:oco} where $M^*$ is substituted by $M_{\B}^*$.

\begin{algorithm}
\label{alg:ocof}
\textbf{Initialize}: $\w_0 = \mathbf{0}$ and $\eta = \frac{U^2}{\sqrt{T}}$\;
\Repeat{the last instance} { At the $t$-th iteration, receive a training instance $\z_t = (\x_t, y_t)\in \bbbr^d\times \{-1, +1\}$.

Update the weight vector such that
\[ \begin{split}
  \w_t &= P_{K}\left[\w_{t-1} - \frac{1}{4RU+1}\cdot\frac{1}{t-1}\sum_{j\in\B_{t-1}}\eta(y_t-y_j)(\x_t - \x_j)\cdot\bbbi_{\ell_{hinge}(\w_{t-1}, \z_t, \z_j) > 0}\right]
\end{split}.
\]
}
\caption{Online Projected Gradient Descent for Bipartite Ranking with Finite Buffer.}
\end{algorithm}

Finally we relate $M_{\B}^*(\w)$ to $M^*(\w)$, which is defined as
\[
M^*(\w) = \nmlz\frac{1}{t-1}\sum_{j=1}^{t-1}\ell(\w, \z_t, \z_j).
\]
It is easy to see that for any fixed $\w$, $M^*(\w)$ is an unbiased estimator of $\R(\w)$. Therefore, if  $M_{\B}^*(\w)$ is close to $M^*(\w)$ with high probability, we can say that $\R(\widehat{\w_{\B}})$ is close to $\R(\w)$ (Notice that $\w$ can be arbitrary) with high probability. We have the following lemma,
\begin{lemma}
  Suppose the pairwise loss $\ell$ satisfies the conditions given in Theorem~\ref{thm:concetra}. Assuming the sequential updating rule for the buffer, then $\forall \epsilon > 0$, we have
  \[
  \Pr\left\{\sup_{\w}\left[\frac{1}{n-2}M^*_{\B}(\w) - \frac{1}{n-c_n}M^*_{cn}(\w)\right]\geqslant \epsilon\right\} \leqslant \mathcal{N}\left(\mathcal{H}, \frac{\epsilon}{8\text{Lip}(\phi)}\right)e^{-c^2(1-c)^2n\epsilon^2}
  \]
\end{lemma}
\begin{proof}
  To prove this, define $\Omega(\w)$ as
\[
\begin{split}
  \Omega(\w) &= \frac{1}{n-2}M^*_{\B}(\w) - \frac{1}{(n-c_n)}M^*_{cn}(\w)\\
   &= \frac{1}{n-2}\sum_{t=2}^{n-1}\frac{1}{|\B_t|}\sum_{j\in\B_t}\ell(\w, \z_t, \z_j) - \nmlz\frac{1}{t-1}\sum_{j=1}^{t-1}\ell(\w,\z_t, \z_j).
\end{split}
\]
Suppose the sequential buffer strategy is used where $\B_t = \{\x_{t-1}, \x_{t-2}, \cdots, \x_{t-|\B| + 1}\}$. We next compute the variation of $\Omega(\w)$ when changing any of its $n$ random variables. For both terms, there are two situations when one substitutes $\x_i$ with $\x_i'$, i.e., when $t=i$ and $t>i$.
\begin{itemize}
  \item For the first term, when $t = i$, its variation is bounded by $1/(n-2)$. When $t>i$, as we use FIFO, $\x_i$ can only be kept in the buffer for $\B$ rounds, thus the variation is bounded by $1/(n-2)$.
  \item For the second term, when $t=i$, the variation bounded by $1/(n-cn)$. We have previously shown in~(\ref{eqn:cjupper}) that changing $\x_i$ when $i > t$ alters the second term by $1/cn$.
\end{itemize}
Consequently, $\Omega(\w)$ is bounded by $2/(c(1-c)n)$ when one variable is changed. It is easy to see $\bbbe[\Omega(\w)] = 0$. By McDiarmid's inequality, we have
\[
\Pr(\Omega(\w) \geqslant \epsilon) \leqslant \exp^{-c^2(1-c)^2n\epsilon^2}.
\]
Using the covering number technique again, we have
\[
    \Pr(\sup_{\w} \Omega(\w) \geqslant \epsilon) \leqslant \mathcal{N}\left(\mathcal{H}, \frac{\epsilon}{8\text{Lip}(\phi)}\right)e^{-c^2(1-c)^2n\epsilon^2}.
\]
\end{proof}

\section{Application: Online Metric Learning}
\label{sec:metric}
In the past decade, metric learning has become an active field in machine learning with numerous applications in information retrieval, classification, etc. In this section, we consider online learning algorithms for supervised metric learning. Generally speaking, supervised metric learning seeks to find a Mahalanobis distance metric that makes instances ``agree with'' their labels. The intuition is that under the desired metric, examples that share the same label should be close while ones from different labels should be far away from each other. The metric is parameterized by a positive semi-definite matrix $\A$ such that for any two example $\z_i, \z_j$, we have
\[
 d_\A(\z_i, \z_j) = \sqrt{(\x_i - \x_j)^T\A(\x_i - \x_j)}.
\]
Under the batch setting, recent work derived the generalization bounds for supervised metric learning~\citep{jin2009regularized,cao2012generalization,bellet2012robustness}. Several online metric learning algorithms have been proposed~\citep{davis2007information,jain2008online,jin2009regularized,kunapuli2012mirror}; these analyzed to provide regret bounds but to date the generalization performance of these algorithms has not been analyzed, possibly because no tools existed to provide online-to-batch conversion with pairwise loss functions. In this section, we provide such an analysis.

We consider the hypothesis space to be the vector space of symmetric semi-definite matrices $\mathbb{S}^+_d$ of size $d\times d$, equipped with the inner product, $\langle \mathbf{X}, \mathbf{Y} \rangle := \textbf{Tr}(\mathbf{X}^T\mathbf{Y})$. $\|\X\|_F^2 = \langle \X, \X \rangle$ denotes the Frobenius norm of matrix $\X$. Following~\cite{jin2009regularized}, we will work with the following pairwise loss function
\[
 \ell(\A, \z_i, \z_j) = g\left(y_{ij}\left[1 - (\x_i - \x_j)^T\A(\x_i - \x_j)\right]\right),
\]
where $g$ is a normalized version of the hinge loss and $y_{ij} = 1$ if $y_i = y_j$ and $-1$ otherwise. With this setting, examples in the same class must have distance 0 to obtain zero loss and examples in different classes must have distance larger than 2 to yield zero loss.

Before further development, we first state and prove the following theorem for online gradient descent over matrices. Its proof is similar to the case where the hypothesis space is $\bbbr^d$~\citep{zinkevich}, but we include a proof for completeness. More sophisticated analysis for learning with matrices is provided by~\cite{kakade2012regularization}.
\begin{theorem}
\label{thm:ocoma}
  Assume that $K\subset \mathbb{S}^+_d$ is convex, closed, non-empty and bounded such that
  \[
    \sup_{\A,\bb}\|\A - \bb\|_F \leqslant U.
  \]
 Assume that at round $t$ we are working with a convex loss function $\ell_t:\mathbb{S}^+_d\rightarrow\bbbr^+$ such that,
 \[
 \|\nabla\ell_t(\A_t)\|_F \leqslant D.
 \]
  Consider an online learner with update rule
 \[
   \A_{t+1} = \P_K\left[\A_{t} - \eta\nabla\ell_t(\A_t)\right],
  \]
  where $\P_K$ is the projection operator. If we set the learning rate $\eta = \frac{U}{D}\sqrt{\frac{1}{T}}$, we have
  \[
    \sum_{t=1}^T\ell_t(\A_t) - \inf_{\bb\in K}  \sum_{t=1}^T\ell_t(\bb) \leqslant UD\sqrt{T}.
  \]
\end{theorem}

\begin{proof}
Since $K\subset \mathbb{S}^+_d$ is convex, closed, non-empty and bounded subspace of a Hilbert space, we have~\citep{rudin2006real}, $\forall\A,\bb\in\mathbb{S}^+_d$,
\[
\|\P_K(\A) - \P_K(\bb)\|^2_F\leqslant \|\A-\bb\|^2_F.
\]
For an arbitrary $\bb\in K$, we have
  \[
  \begin{split}
    &\|\A_{t+1} - \bb\|^2_F - \|\A_{t} - \bb\|^2_F \\
    &\qquad = \|\P_K\left[\A_{t} - \eta\nabla\ell_t(\A_t)\right] - \bb\|^2_F - \|\A_{t} - \bb\|^2_F \qquad (\because \|\P_K(\X)\|^2_F\leqslant \|\X\|^2_F)\\
    &\qquad \leqslant \|\A_{t} - \bb - \eta\nabla\ell_t(\A_t) \|^2_F - \|\A_{t} - \bb\|^2_F\\
    &\qquad= \|\eta\nabla\ell_t(\A_t) \|^2_F - 2\eta\langle\nabla\ell_t(\A_t), \A_t - \bb\rangle\\
  \end{split}
  \]
  which gives
  \[\langle\nabla\ell_t(\A_t), \A_t - \bb\rangle \leqslant \frac{1}{2\eta}\left(\|\A_{t+1} - \bb\|^2_F - \|\A_{t} - \bb\|^2_F + \eta^2\|\nabla\ell_t(\A_t) \|^2_F\right).\]
  Therefore,
  \[
  \begin{split}
    \sum_{t=1}^T\ell_t(\A_t) - \sum_{t=1}^T\ell_t(\bb) &\leqslant \sum_{t=1}^T\langle \nabla\ell_t(\A_t), \A_t - \bb\rangle\\
    &\leqslant \frac{1}{2\eta}\sum_{t=1}^T\left(\|\A_{t+1} - \bb\|^2_F - \|\A_{t} - \bb\|^2_F\right) + \frac{\eta}{2}D^2T\\
    &\leqslant \frac{1}{2\eta}U^2 + \frac{\eta}{2}D^2T.
  \end{split}
  \]
  Setting the learning rate $\eta = \frac{U}{D}\sqrt{\frac{1}{T}}$ yields the result.
\end{proof}

For the metric learning, at each round the loss function is
\[
\ell_t(\A) = \frac{1}{t-1}\sum_{j=1}^{t-1}\ell_j^t(\A),
\]
where
\[
\begin{split}
\ell^t_j(\A) &= \left[1 - y_{tj}(1 - (\x_t - \x_j)^T\A(\x_t - \x_j))\right]_+\\
&= \left[1 - y_{tj}(1 -\langle \A, \X_{tj}) \rangle\right]_+,
\end{split}
\]
where $y_{tj} \in \{+1, -1\}$ and $\X_{tj} = (\x_t - \x_j)(\x_t - \x_j)^T$. It is easy to see that $\ell^t_j$ is a convex function of $\A$.

%For any $\lambda \in (0,1)$, we have
%\[
%\begin{split}
%  \ell (\lambda\A + (1-\lambda)\bb) & = \left[1 - c(1 -\langle \lambda\A + (1-\lambda)\bb, \X \rangle)\right]_+\\
%  &= \left[c\lambda\langle \A, \X \rangle +  c(1-\lambda)\langle \bb, \X \rangle + 1 - c\right]_+\\
%  &= \left[c\lambda\langle \A, \X \rangle + \lambda(1-c) +  c(1-\lambda)\langle \bb, \X \rangle + (1-\lambda)(1 - c)\right]_+\\
%  &= \left[\lambda(1-c(1-\langle \A, \X \rangle)) + (1-\lambda)\lambda(1-c(1-\langle \bb, \X \rangle))\right]_+\\
%  &\leqslant \lambda\left[\lambda(1-c(1-\langle \A, \X \rangle))\right]_+ + (1-\lambda)\left[\lambda(1-c(1-\langle \bb, \X \rangle))\right]_+\\
%  &= \lambda\ell(\A) + (1-\lambda)\ell(\B).
%\end{split}
%\]
Next we choose the convex set $K$ to be \[K=\{\A: \A\in\mathbb{S}^+_d, \|\A\|_F \leqslant U\}.\] We assume $\sup_t\|\x\|_2\leqslant R$. To bound $\ell_t$ in $[0,1]$, we redefine $\ell_j^t$ scaling it by a factor of $\frac{1}{2 + UR^2}$. To utilize Theorem~\ref{thm:ocoma}, we need to bound the subgradient of $\ell^t_j$ as follows,
\[
\|\nabla\ell_j^t(\A)\| \leqslant \frac{1}{2+UR^2}\|y_{tj}\X_{ij}\|_F \leqslant \frac{R^2}{2+UR^2} \leqslant \frac{1}{U}.
\]

\begin{algorithm}
\label{alg:ocometric}
\textbf{Initialize}: $\eta = \frac{U^2}{\sqrt{T}}$ and $\A_0$ to be any PSD matrix with $\|\A_0\|_F\leqslant U$. \\
\Repeat{the last instance} { At the $t$-th iteration, receive a training instance $\z_t = (\x_t, y_t)\in \bbbr^d\times \{-1, +1\}$.

Update the weight vector such that
\[ \begin{split}
  \A_t &= \P_{K}\left[\A_{t-1} - \frac{1}{2+UR^2}\cdot\frac{1}{t-1}
  \sum_{j=1}^{t-1}\eta y_{tj}\X_{tj}\cdot\bbbi_{\ell^t_j(\A_{t-1})>0}\right]
\end{split}.
\]
}
\caption{Online Projected Gradient Descent for Metric Learning.}
\end{algorithm}

Finally, as in previous results, define $M^* = \inf_\bb\sum_t\ell_t(\bb)$. Using~(\ref{eqn:convex}) we bound the risk of Algorithm~\ref{alg:ocometric} with the following theorem.
\begin{theorem}
\label{thm:ocometric}
   Let $\A_0, \cdots, \A_{n-1}$ be the ensemble of hypotheses generated by Algorithm~\ref{alg:ocometric} and let
   \[
     \bar{\A} = \nmlz\A_t
   \]
then the probability that
\[
\R(\bar{\A}) \geqslant  \frac{1}{n-c_n}\left(M^* + \sqrt{n}\right) + \epsilon
\]
is at most
\[
\left[2\left(\frac{32R}{\epsilon} + 1\right)^d + 1\right]
\exp\left\{-\frac{(cn-1)\epsilon^2}{64} + \ln n\right\}.
\]
\end{theorem}
The key computational step is to project the matrix to $\mathbb{S}^+_d$. \cite{jin2009regularized} showed an efficient way to perform the projection. Notice that Theorem~\ref{thm:ocometric} also applies to the online learning algorithm proposed in~\cite{jin2009regularized} with proper normalization. In general, we believe our results are applicable for all online metric learning algorithms with proved regret guarantees.

\section{Conclusion and Future work}
\label{sec:conclu}
In this paper, we provide generalization bounds for online learners using pairwise loss functions and apply these to bipartite ranking and supervised metric learning. These are the first results to provide online-to-batch conversion for pairwise loss functions and as we demonstrate, they are applicable to multiple problems.

There are several directions for possible future work. From an empirical perspective, although the random Online AUC Maximization (OAM) is simple and
easy to implement, it seems that it does not maintain buffers in an optimal
way. Intuitively, one might want to store important examples
 that help build the correct ranker instead of
using a random buffer. We are currently exploring ideas on building a
smart buffer to improve its performance.

From the theoretical point of view, one direction is to improve the current bounds to achieve faster convergence rates. Another direction is deriving tighter mistake bounds for random OAM. Finally, our buffered based results require a buffer size of $\mathcal{O}(\log n)$ to guarantee convergence. It would be interesting to investigate whether this is necessary, or alternatively improve the results to show convergence for constant size buffer. It is also interesting to extend our results to dependent data, for example, assuming the data to be stationary mixing sequences~\citep{agarwal2011generalization} or utilizing new tools such as the sequential Rademacher complexity~\citep{sasha2013}.

%Alternatively, one can analyze Algorithm~\ref{alg:oam}
%from a totally different point of view. Under the batch
%setting,~\cite{clemencon2008ranking,rejchel2012ranking} already provide fast convergence
%rates for the \emph{empirical risk minimizer}.  %Since
%Algorithm~\ref{alg:oam} is in fact a stochastic gradient descent algorithm
%to minimize the $U$-statistic,
%using online convex programming
%techniques~\citep{zinkevich, shalev2007online}, one can show that the
%regret is small. Combining this with the batch results automatically
%yields risk bounds for the algorithm. It is interesting to compare
%this approach to the one proposed in this paper in terms of the
%risk bounds that can be obtained.
%However, it is important to note
%that the approach in this paper is more general in two ways. First we
%only assume that the loss function is Lipschitz instead of convex.
%Second, the ensemble of hypotheses can be
%produced by an \emph{arbitrary} online learning algorithm, not just
%stochastic gradient descent.

\acks{YW and DP thank Nicol\`{o} Cesa-Bianchi for early discussions. YW and RK were partly supported by NSF grant
IIS-0803409. Part of this research was done when YW was an intern at Akamai Technologies in 2011.}
\bibliography{online}

\appendix
\newpage
\section{Complete Proof of Claim~\ref{clm:sym}}
\begin{proof}[\emph{Proof of Claim~\ref{clm:sym}}]
The required probability can be bounded as follows.

{\small{
\begin{align}
\label{eqn:symsplit}
&\Pr_{Z^n\sim\D^n, \Xi^n\sim\D^n}\left(\nmlz \left[\bbbe_t[\widetilde{M}_t] - \bbbe_t[M_t]\right]\geqslant \frac{\epsilon}{2}\right)\nonumber\\
&\geqslant
\Pr_{Z^n\sim\D^n, \Xi^n\sim\D^n}\Bigg(\left\{\nmlz\left(\mathcal{R}(h_{t-1}) - \bbbe_t[M_t]\right) \geqslant \epsilon\right\}\nonumber\\
&\qquad\qquad\qquad\qquad\qquad\bigcap
\left\{\bigg|\nmlz\left[\bbbe_t[\widetilde{M}_t] - \mathcal{R}(h_{t-1})\right]\bigg| \leqslant \frac{\epsilon}{2}\right\}\Bigg)\nonumber\\
&= \bbbe_{Z^n\sim\D^n, \Xi^n\sim\D^n}\left[\bbbi_{\left\{\nmlz \left(\mathcal{R}(h_{t-1}) - \bbbe_t[M_t]\right) \geqslant \epsilon\right\}}
\times \bbbi_{\left\{\big|\nmlz\left[\bbbe_t[\widetilde{M}_t] - \mathcal{R}(h_{t-1})\right]\big| \leqslant \frac{\epsilon}{2}\right\}}\right]\nonumber\\
&= \bbbe_{Z^n\sim\D^n}\Bigg[\bbbe_{\Xi^n\sim\D^n}\bigg[\bbbi_{\left\{\nmlz \left(\mathcal{R}(h_{t-1}) - \bbbe_t[M_t]\right) \geqslant \epsilon\right\}}
\times \bbbi_{\left\{\big|\nmlz\left[\bbbe_t[\widetilde{M}_t] - \mathcal{R}(h_{t-1})\right]\big| \leqslant \frac{\epsilon}{2}\right\}}\bigg|Z^n\bigg]\Bigg]\nonumber\\
&=
\bbbe_{Z^n\sim\D^n}\Bigg[\bbbi_{\left\{\nmlz \left(\mathcal{R}(h_{t-1}) - \bbbe_t[M_t]\right) \geqslant \epsilon\right\}}\nonumber\\
&\qquad\qquad\qquad\qquad\qquad\times \Pr_{\Xi^n\sim\D^n}\left(\bigg|
\nmlz\left[\bbbe_t[\widetilde{M}_t] - \mathcal{R}(h_{t-1})\right] \bigg|\leqslant \frac{\epsilon}{2}\bigg|Z^n\right)\Bigg]
\end{align}}}

We
next show that for sufficiently large $n$,
\[
\Pr_{\Xi^n\sim\D^n}\left( \bigg|\nmlz\left[\bbbe_t[\widetilde{M}_t] -
\mathcal{R}(h_{t-1})\right]\bigg| \leqslant
\frac{\epsilon}{2}\bigg| Z^n\right) \geqslant \frac{1}{2},
\]
which combined with (\ref{eqn:symsplit}) implies (\ref{eqn:symm}). To begin with, we first show that the corresponding random variable has mean zero
\[
\begin{split}
&\bbbe_{\Xi^n\sim\D^n}\left(\nmlz\left[\bbbe_t[\widetilde{M}_t] - \mathcal{R}(h_{t-1})\right] \bigg| Z^n \right) \\
&\quad =\bbbe_{\Xi^n\sim\D^n}\left( \nmlz\bbbe_t\bigg[\frac{1}{t-1}\sum_{j = 1}^{t-1} \ell(h_{t-1}, \z_t, \bxi_j)\bigg] - \mathcal{R}(h_{t-1}) \bigg| Z^n\right)\\
&\quad = \nmlz\left[\frac{1}{t-1}\sum_{j = 1}^{t-1}\bbbe_{\bxi_j}\bbbe_t[\ell(h_{t-1}, \z_t, \bxi_j)|Z^t] - \mathcal{R}(h_{t-1})\right]\\
&\quad = \nmlz\left[\left(\frac{1}{t-1}\sum_{j =
1}^{t-1}\mathcal{R}(h_{t-1})\right) - \mathcal{R}(h_{t-1})\right] =
0.
\end{split}
\]
Thus, we can use Chebyshev's inequality to bound the conditional probability as follows
\begin{displaymath}
\begin{split}
\Pr\left( \left\{\bigg|\nmlz\bigg[\bbbe_t[\widetilde{M}_t] - \mathcal{R}(h_{t-1})\bigg]\bigg| \leqslant \frac{\epsilon}{2}\right\} \bigg|Z^n\right) \geqslant 1 - \frac{\textbf{Var}\left\{\nmlz\bbbe_t[\widetilde{M}_t]\right\}}{\epsilon^2/4}. %\geqslant 1 - \frac{4}{n\epsilon^2} \geqslant \frac{1}{2},
\end{split}
\end{displaymath}
%whenever $n\epsilon^2 \geqslant 8$. The variance inequality i.e. $\sigma^2 \leqslant 1/4n$
To bound the variance, we resort to the following
Theorem~\citep[see][Theorem 9.3]{devroye1996probabilistic}
\begin{theorem}
Let $X_1, \cdots, X_n$ be independent random variables taking values in a set $A$, and assume that $f:A^n\rightarrow\bbbr$ satisfies
\begin{displaymath}
\underset{\x_1,\x_2, \x_n, \x'}{\sup}\left| f(\x_1, \cdots, \x_i, \cdots,  \x_n) - f(\x_1, \cdots, \x', \cdots, \x_n)\right| \leqslant c_i\quad \forall 1\leqslant i\leqslant n.
\end{displaymath}
Then
\begin{displaymath}
\textbf{Var}(f(X_1,\cdots, X_n)) \leqslant \frac{1}{4}\sum_{i=1}^nc_i^2.
\end{displaymath}
\end{theorem}

To bound the variance, we first investigate the largest variation when changing one random variable $\bxi_j$ with others fixed.
From (\ref{eqn:tm}), it can be easily seen that changing any of
the $\bxi_j$ varies each $\bbbe_t[\widetilde{M}_t]$, where $t > j$ by at most by $1/(t-1)$. Remember we are only concerned with $\bbbe_t[\widetilde{M}_t]$ when $t \geqslant c_n$.
Therefore, we can see that the variation of $\nmlz\bbbe_t[\widetilde{M}_t]$
regarding the $j$th example $\bxi_j$ is bounded by
\[
c_j = \frac{1}{n-c_n}\left[\sum_{t = (j \vee c_n + 1)}^{n-1}\frac{1}{t - 1}\right]
\leqslant \frac{1}{n-c_n}\left[\sum_{t = c_n + 1}^{n-1}\frac{1}{c_n}\right]  \leqslant \frac{1}{cn}.
\]
Thus, by Theorem~9.3 in \citep{devroye1996probabilistic}, we have
\begin{equation}
%\label{eqn:sumci}
\begin{split}
\textbf{Var}\left(\nmlz\bbbe_t[\widetilde{M}_t]\right) &\leqslant
\frac{1}{4}\sum_{i=1}^n c_i^2  \leqslant \frac{1}{4c^2n}.
\end{split}
\end{equation}
Thus, whenever $\epsilon^2c^2n > 2$, the LHS of (\ref{eqn:cheyb}) is greater or equal than $1/2$. This completes the proof of Claim~\ref{clm:sym}.
%
%From (\ref{eqn:tm}), it can be easily seen that changing any of
%the $\bxi_j$ varies each $\bbbe_t[\widetilde{M}_t]$, where $t > j$ by at most by $1/(t-1)$.
%Therefore, we can see that the variation of $\nmlz\bbbe_t[\widetilde{M}_t]$
%regarding the $j$th example $\bxi_j$ is bounded by
%\[
%c_j = \frac{1}{n-c_n}\left[\sum_{t = j + 1}^n\frac{1}{t - 1}\right]
%= \frac{1}{n-c_n}\left[\sum_{t = j}^{n-1}\frac{1}{t}\right] =
%\frac{1}{n-c_n}H_j(n).
%\]
%The partial sum of the harmonic series $H_j(n) \leqslant \log(n),\ \forall j > 2$.  Thus, we have
%\[
%\textbf{Var}\left(\nmlz\bbbe_t[\widetilde{M}_t]\right) \leqslant
%\frac{1}{4}\sum_{i=1}^n c_i^2  \leqslant \frac{1}{4(1-c)^2} \frac{\log^2(n)}{n}.
%\]
%Thus, whenever $n/\log^2(n)\geqslant 2/(\epsilon(1-c))^2$, the LHS of (\ref{eqn:cheyb}) is greater or equal than $1/2$. This completes the proof of Claim~\ref{clm:sym}.
\end{proof}

\section{Proof of Lemma~\ref{lemma:lt}}

\begin{proof}[\emph{Proof of Lemma~\ref{lemma:lt}}]
To bound the probability, we use the McDiarmid's inequality.
\begin{theorem}[McDiarmid's Inequality]
  Let $X_1, \cdots, X_N$ be independent random variables with $X_k$ taking values in a set $A_k$ for each $k$. Let $\phi:(A_1\times \cdots \times A_N)\rightarrow\bbbr$ be such that
  \[
  \underset{x_i\in A_i, x'_k\in A_k}{\sup}|\phi(x_1, \cdots, x_N) - \phi(x_1, \cdots, x_{k-1}, x'_k, x_{k+1},\cdots, x_N)| \leqslant c_k.
  \]
  Then for any $\epsilon > 0$,
  \[ \Pr\left\{\phi(x_1,\cdots, x_N) - \bbbe\phi(x_1,\cdots, x_N)\geqslant \epsilon\right\} \leqslant e^{-2\epsilon^2/\sum_{k=1}^Nc_k^2}, \]
  and
  \[ \Pr\left\{|\phi(x_1,\cdots, x_N) - \bbbe\phi(x_1,\cdots, x_N)|\geqslant \epsilon\right\} \leqslant 2e^{-2\epsilon^2/\sum_{k=1}^Nc_k^2}. \]
\end{theorem}
For any fixed $f\in\mathcal{H}$, we have
\[
\begin{split}
  \bbbe_{\z_{1:t-1}, \bxi_{1:t-1}}\left[L_t(f)\right]
  &= \frac{1}{t-1}\sum_{j = 1}^{t-1}\bbbe_{\z_{1:t-1}, \bxi_{1:n-1}}\bbbe_{\z}\left[\ell(f, \z, \bxi_j) - \ell(f, \z, \z_j)\right]\\
  &=  \frac{1}{t-1}\sum_{j = 1}^{t-1}\left(\bbbe_{\bxi_j}\bbbe_{\z}[\ell(f, \z, \bxi_j)] - \bbbe_{\z_j}\bbbe_{\z}[\ell(f, \z, \z_j)]\right) = 0\\
\end{split}
\]
Now, $L_t(f)$ is a function of $2(t-1)$ variables with each affecting its value at most by $c_i = 1/(t-1), i = 1, 2, \cdots, 2(t-1)$. Thus, we have
  $
  \sum_{i = 1}^{2(t-1)} c_i^2  = \frac{2}{t-1}.
  $
Finally, using the McDiarmid's inequality, we get
\[
  \Pr_{Z^t\sim\D^t, \Xi^t\sim\D^t}\left(L_t(f)\geqslant\epsilon\right) \leqslant
\exp\left\{-(t-1)\epsilon^2\right\}.
\]
%and this completes the proof of Lemma~\ref{lemma:lt}.
\end{proof}
\section{Proof of Lemma~\ref{lemma:ltv}}
\begin{proof}[\emph{Proof of Lemma~\ref{lemma:ltv}}]  From the definition of $L_t$ and the assumption on $\phi$  we have

  \[
  \begin{split}
    L_t(h_1) - L_t(h_2) &= \frac{1}{t-1}\sum_{j=1}^{t-1}
    \bigg[\bbbe_{\z}[\ell(h_1, \z, \bxi_j) - \ell(h_1, \z, \z_j)] - \bbbe_{\z}[\ell(h_2, \z, \xi_j) - \ell(h_2, \z, \z_j)]\bigg]\\
    &=\frac{1}{t-1}\sum_{j=1}^{t-1}\bbbe_{\z}\bigg\{ \big[\phi(y - \tilde{y}_j, h_1(\x,\tilde{\x}_j)) - \phi(y - y_j, h_1(\x,\x_j))\big] \\
    &\qquad -  \big[\phi(y - \tilde{y}_j, h_2(\x,\tilde{\x}_j)) - \phi(y - y_j, h_2(\x,\x_j))\big]\bigg\} \\
    &=\frac{1}{t-1}\sum_{j=1}^{t-1}\bbbe_{\z}\bigg\{ \big[\phi(y - \tilde{y}_j, h_1(\x,\tilde{\x}_j)- \phi(y - \tilde{y}_j, h_2(\x,\tilde{\x}_j)\big] \\
    &\qquad -  \big[\phi(y - y_j, h_1(\x,\x_j)) - \phi(y - y_j, h_2(\x,\x_j))\big]\bigg\} \\
    &\leqslant \frac{1}{t-1}\sum_{j=1}^{t-1}\bbbe_{\z}\bigg\{\text{Lip}(\phi)\bigg|h_1(\x,\tilde{\x}_j)
     - h_2(\x,\tilde{\x}_j)\bigg| + \text{Lip}(\phi)\bigg|h_1(\x,{\x}_j) - h_2(\x,{\x}_j)\bigg|\bigg\}\\
    &\leqslant \frac{1}{t-1}\sum_{j=1}^{t-1}\left[2\text{Lip}(\phi)\ \underset{\x', \x''}{\sup}\left|h_1(\x',\x'') -
   h_2(\x',\x'')\right|\right]= 2\text{Lip}(\phi)\|h_1 - h_2\|_\infty
  \end{split}.
  \]

%  \[
%  \begin{split}
%    L_t(h_1) - L_t(h_2) &= \frac{1}{t-1}\sum_{j=1}^{t-1}
%    \bigg[\bbbe_{\z}[\ell(h_1, \z, \bxi_j) - \ell(h_1, \z, \z_j)] - \bbbe_{\z}[\ell(h_2, \z, \xi_j) - \ell(h_2, \z, \z_j)]\bigg]\\
%    &=\frac{1}{t-1}\sum_{j=1}^{t-1}\bbbe_{\z}\bigg\{ \big[\phi(y - \tilde{y}_j, h_1(\x) - h_1(\tilde{\x}_j)) - \phi(y - y_j, h_1(\x) - h_1(\x_j))\big] \\
%    &\qquad -  \big[\phi(y - \tilde{y}_j, h_2(\x) - h_2(\tilde{\x}_j)) - \phi(y - y_j, h_2(\x) - h_2(\x_j))\big]\bigg\} \\
%    &=\frac{1}{t-1}\sum_{j=1}^{t-1}\bbbe_{\z}\bigg\{ \big[\phi(y - \tilde{y}_j, h_1(\x) - h_1(\tilde{\x}_j)- \phi(y - \tilde{y}_j, h_2(\x) - h_2(\tilde{\x}_j)\big] \\
%    &\qquad -  \big[\phi(y - y_j, h_1(\x) - h_1(\x_j)) - \phi(y - y_j, h_2(\x) - h_2(\x_j))\big]\bigg\} \\
%    &\leqslant \frac{1}{t-1}\sum_{j=1}^{t-1}\bbbe_{\z}\bigg\{\text{Lip}(\phi)\bigg|\bigg[h_1(\x) - h_1(\tilde{\x}_j)\bigg] - \bigg[h_2(\x) - h_2(\tilde{\x}_j)\bigg]\bigg|\\
%    &\qquad + \text{Lip}(\phi)\bigg|\bigg[h_1(\x) - h_1({\x}_j)\bigg] - \bigg[h_2(\x) - h_2({\x}_j)\bigg]\bigg|\bigg\}\\
%    &\leqslant \frac{1}{t-1}\sum_{j=1}^{t-1}\left[4\text{Lip}(\phi)\ \underset{\x'}{\sup}\left|h_1(\x') -
%   h_2(\x')\right|\right]= 4\text{Lip}(\phi)\|h_1 - h_2\|_\infty
%  \end{split}.
%  \]
\end{proof}
\section{Proof for the Risk Bound of Convex Losses in Section~\ref{sec:convex}}
\begin{proof}
  Using Jensen's inequality, we have
  \[
  \begin{split}
    \mathcal{R}(\bar{h}) &= \bbbe_{\z}\bbbe_{\z'}\left[\phi\left(y - y', \nmlz h_{t-1}(\x) - \nmlz h_{t-1}(\x')\right)\right]\\
    &= \bbbe_{\z}\bbbe_{\z'}\left[\phi\left(y - y', \nmlz\left[h_{t-1}(\x) - h_{t-1}(\x')\right]\right)\right]\\
    &\leqslant \nmlz\bbbe_{\z}\bbbe_{\z'}[\phi(y - y',  h_{t-1}(\x) -  h_{t-1}(\x'))]\\
    &= \nmlz\bbbe_{\z}\bbbe_{\z'}[\ell(h_{t-1}, \z, \z')] = \nmlz\mathcal{R}(h_{t-1}).
  \end{split}
  \]
  Combining with Theorem~\ref{thm:concetra}, we have
  \[
\Pr\left(\mathcal{R}(\bar{h}) \geqslant M^n(Z^n) +
\epsilon\right) \leqslant \left[2\mathcal{N}\left(\mathcal{H},
\frac{\epsilon}{16\text{Lip}(\phi)}\right) + 1\right]
\exp\left\{-\frac{(cn-1)\epsilon^2}{64} + \ln n\right\}.
\]
\end{proof}
\section{Proof of Theorem~\ref{thm:main}}
\begin{proof}[\emph{Proof of Theorem~\ref{thm:main}}]
The proof is adapted from the proof for Theorem~4 in~\citep{cesa2004generalization}. The main difference is that instead of using the Chernoff bound we use a large deviation bound for the $U$-statistic as follows.
\begin{lemma}{~\citep[see][Appendix]{clemencon2008ranking}}
Suppose we have i.i.d. random variables $X_1, \cdots, X_n\in\mathcal{X}$ and the $U-$statistic is defined as
\[
U_n = \frac{1}{n(n-1)}\sum_{i\neq j}^nq(X_i, X_j) = \frac{2}{n(n-1)}\sum_{i > j}^nq(X_i, X_j),
\]
where the kernel $q:\mathcal{X}\times \mathcal{X} \rightarrow \bbbr$ is a symmetric real-valued function. Then we have,
\begin{equation}
\label{eqn:ularge}
\Pr(| U_n - \bbbe[U_n] | \geqslant \epsilon) \leqslant 2\exp\{-(n-1)\epsilon^2\}.
\end{equation}
\end{lemma}
Therefore, by (\ref{eqn:ularge}), we have
\[
\Pr\left(|\widehat{\mathcal{R}}(h_t, t+1) - \mathcal{R}(h_t)| \geqslant \epsilon\right) \leqslant 2\exp\{-(n-t-1)\epsilon^2\},
\]
or equivalently,
\begin{equation}
\label{eqn:ulargeconf}
\Pr\left(\bigg|\widehat{\mathcal{R}}(h_t, t+1) - \mathcal{R}(h_t)\bigg| \geqslant \sqrt{\frac{1}{n - t -1}\ln\frac{2}{\delta}}\right) \leqslant \delta.
\end{equation}
By the definition of $c_\delta$ and (\ref{eqn:ulargeconf}), one can see that
\begin{equation}
  \label{eqn:conf}
  \Pr(\big|\widehat{\mathcal{R}}(h_t, t+1) - \mathcal{R}(h_t)\big| > c_\delta(n-t)) \leqslant \frac{\delta}{(n - c_n)(n - c_n + 1)}.
\end{equation}
Next, we show the following lemma,
\begin{lemma}
\label{lemma:sigma}
    Let $h_0, \cdots, h_{n-1}$ be the ensemble of hypotheses generated by an arbitrary online algorithm $\mathcal{A}$ working with a pairwise loss $\ell$ which satisfies the conditions given in Theorem~\ref{thm:concetra}. Then for any $0 < \delta \leqslant 1$,  we have
    \begin{equation}
    \label{eqn:ms1}
    \Pr\left(\mathcal{R}(\widehat{h}) \geqslant \underset{c_n - 1\leqslant t < n - 1}{\min}\left(\mathcal{R}(h_t) + 2c_\delta(n-t)\right) \right)\leqslant \delta.
    \end{equation}
\end{lemma}
\begin{proof}[\emph{Proof of Lemma~\ref{lemma:sigma}}]
  The proof closely follows the proof of Lemma 3 in~\citep{cesa2004generalization} and is given for the sake of completeness. Let
  \[
    T^* = \underset{c_n - 1\leqslant t < n - 1}{\text{argmin}}\left(\mathcal{R}(h_t) + 2c_\delta(n-t)\right),
  \]
  and $h^* = h_{T^*}$ is the corresponding hypothesis that minimizes the penalized true risk and let $\widehat{\mathcal{R}}^*$ to be the penalized empirical risk of $h_{T^*}$, i.e.
  \[
    \widehat{\mathcal{R}}^* = \widehat{\mathcal{R}}(h_{T^*}, T^*+1).
  \]
  Set, for brevity
  \[
    \widehat{\mathcal{R}}_t = \widehat{\mathcal{R}}(h_t, t+1),
  \]
  and let
  \[
    \widehat{T} = \underset{c_n - 1\leqslant t < n - 1}{\text{argmin}}(\widehat{\mathcal{R}}_t + c_\delta(n-t)),
  \]
  where $\widehat{h}$ defined in (\ref{eqn:hh}) coincides with $h_{\widehat{T}}$. With this notation, and since
  \[
  \widehat{\mathcal{R}}_{\widehat{T}} + c_\delta(n - \widehat{T}) \leqslant \widehat{\mathcal{R}}^* + c_\delta(n - T^*)
  \]
  holds with certainty, we can write
  \[
  \begin{split}
   &\Pr\left(\R(\widehat{h}) > \R(h^*) + \mathcal{E}\right) \\
   &\qquad =\Pr\left(\R(\widehat{h}) > \R(h^*) + \mathcal{E},\ \widehat{\mathcal{R}}_{\widehat{T}} + c_\delta(n - \widehat{T}) \leqslant \widehat{\mathcal{R}}^* + c_\delta(n - T^*)\right)\\
   &\qquad \leqslant \sum_{t = c_n - 1}^{n-2}\Pr\left(\R(h_t) > \R(h^*) + \mathcal{E},\ \widehat{\mathcal{R}}_{t} + c_\delta(n - t) \leqslant \widehat{\mathcal{R}}^* + c_\delta(n - T^*)\right)
  \end{split}
  \]
  where $\mathcal{E}$ is a positive-valued random variable to be specified. Now if
  \[
    \widehat{\mathcal{R}}_{t} + c_\delta(n - t) \leqslant \widehat{\mathcal{R}}^* + c_\delta(n - T^*)
  \]
  holds, then at least one of the following three conditions:
  \[
  \begin{split}
    &\ER_t \leqslant \R(h_t) - c_\delta(n-t)\\
    &\ER^* > \R(h^*) + c_\delta(n - T^*)\\
    &\R(h_t) - \R(h^*) < 2c_\delta(n-T^*)
  \end{split}
  \]
  must hold. Therefore, for any fixed $t$, we can write
  \[
  \begin{split}
    &\Pr\left(\R(h_t) > \R(h^*) + \mathcal{E},\ \widehat{\mathcal{R}}_{t} + c_\delta(n - t) \leqslant \widehat{\mathcal{R}}^* + c_\delta(n - T^*)\right) \\
    &\qquad \leqslant \Pr\left(\ER_t \leqslant \R(h_t) - c_\delta(n-t)\right) + \Pr\left(\ER^* > \R(h^*) + c_\delta(n - T^*)\right) \\
    &\qquad\quad + \Pr\left(\R(h_t) - \R(h^*) < 2c_\delta(n-T^*), \R(h_t) > \R(h^*) + \mathcal{E}\right).
  \end{split}
  \]
  The last term is zero if we choose $\mathcal{E} = 2c_\delta(n-T^*)$. Hence, we can write
  \[
  \small
  \begin{split}
    &\Pr\left(\R(\widehat{h}) > \R(h^*) + 2c_\delta(n-T^*)\right) \\
    &\qquad\leqslant \sum_{t = c_n - 1}^{n-2}\Pr\left(\ER_t \leqslant \R(h_t) - c_\delta(n-t)\right) + (n-c_n)\Pr\left(\ER^* > \R(h^*) + c_\delta(n - T^*)\right)\\
    &\qquad\leqslant (n - c_n)\times \frac{\delta}{(n - c_n)(n - c_n + 1)} \qquad (\text{By (\ref{eqn:conf}).})\\
    &\qquad\quad + (n - c_n)\left[\sum_{t = c_n-1}^{n-2}\Pr\left(\ER_t > \R(h_t) + c_\delta(n-t)\right)\right]\\
    &\qquad \leqslant \frac{\delta}{n - c_n + 1} + (n - c_n)^2\times \frac{\delta}{(n - c_n)(n - c_n + 1)} \qquad (\text{By (\ref{eqn:conf}).})\\
    &\qquad = \frac{\delta}{n - c_n + 1} + \frac{(n-c_n)\delta}{n - c_n + 1} = \delta.
  \end{split}
  \]
\end{proof}
Therefore, we know that
\begin{equation}\label{eqn:confbd}\Pr\left(\mathcal{R}(\widehat{h}) \geqslant \underset{c_n - 1\leqslant t < n - 1}{\min}\left(\mathcal{R}(h_t) + 2c_\delta(n-t)\right) \right)\leqslant \delta.\end{equation}
The next step is to show that with high probability $\underset{c_n - 1\leqslant t < n - 1}{\min}\left(\mathcal{R}(h_t) + 2c_\delta(n-t)\right)$ is close to $M^n$.
To begin with, notice that
\[
\begin{split}
  &\underset{c_n - 1\leqslant t < n - 1}{\min}\left(\R(h_t) + 2c_\delta(n - t)\right)\\
  &\qquad=\underset{c_n - 1\leqslant t < n - 1}{\min}\ \underset{t\leqslant i < n - 1}{\min}\left(\R(h_i) + 2c_\delta(n - i)\right)\\
  &\qquad\leqslant\underset{c_n - 1\leqslant t < n - 1}{\min}\ \nmlzi{1}\left(\R(h_i) + 2c_\delta(n - i)\right)\\
  &\qquad=\underset{c_n - 1\leqslant t < n - 1}{\min}\bigg(\nmlzi{1}\R(h_i) \\
  &\qquad\qquad\qquad\qquad + \nmlzi{2}\sqrt{\frac{1}{n - i - 1}\ln\frac{2(n - c_n)(n - c_n + 1)}{\delta}}\bigg)\\
  &\qquad\leqslant\underset{c_n - 1\leqslant t < n - 1}{\min}\bigg(\nmlzi{1}\R(h_i) + \nmlzi{2}\sqrt{\frac{2}{n - i - 1}\ln\frac{2(n - c_n + 1)}{\delta}}\bigg)\\
  &\qquad\leqslant \underset{c_n - 1\leqslant t < n - 1}{\min}\bigg(\nmlzi{1}\R(h_i) + 4\sqrt{\frac{2}{n - t - 1}\ln\frac{2(n - c_n + 1)}{\delta}}\bigg)
\end{split}
\]
where the last equality holds because $\sum_{i=1}^{n-t-1}\sqrt{1/i} \leqslant 2\sqrt{n-t-1}$ \citep[see][Sec.~2.B]{cesa2004generalization}. Define
\[
M_{m,n} = \frac{1}{n - m}\sum_{t = m}^{n -1} M_t(Z^t).
\]

From Theorem~\ref{thm:concetra}, one can see that for each $t = c_n - 1, \cdots, n-2$,

\begin{equation}
\label{eqn:gen1}
  \Pr\left(\nmlzt\mathcal{R}(h_{i}) \geqslant M_{t,n} +
\epsilon\right) \leqslant \left[2\mathcal{N}\left(\mathcal{H},
\frac{\epsilon}{16\text{Lip}(\phi)}\right) + 1\right]
\exp\left\{-\frac{(t-1)\epsilon^2}{64} + \ln n\right\}.
\end{equation}

Then, set for brevity,
\[
K_t = M_{t,n} + 4\sqrt{\frac{2}{n - t - 1}\ln\frac{2(n - c_n + 1)}{\delta}} + \epsilon.
\]
Using the fact that if $\min(a_1, a_2) \leqslant \min(b_1, b_2)$ then either $a_1\leqslant b_1$ or $a_2\leqslant b_2$, we can write

\[
\begin{split}
 &\Pr\left(\underset{c_n - 1\leqslant t < n - 1}{\min}\left(\R(h_t) + 2c_\delta(n - t)\right)\geqslant \underset{c_n - 1\leqslant t < n - 1}{\min}\ K_t\right)  \\
 &\qquad\leqslant \Pr\Bigg(\underset{c_n - 1\leqslant t < n - 1}{\min}\bigg(\nmlzi{1}\R(h_i) \\
 &\qquad\qquad\qquad + 4\sqrt{\frac{2}{n - t - 1}\ln\frac{2(n - c_n + 1)}{\delta}}\bigg)\geqslant \underset{c_n -1\leqslant t < n - 1}{\min}\ K_t\Bigg)\\
 &\qquad\leqslant \sum_{t = c_n - 1}^{n-2}\Pr\left(\nmlzi{1}\R(h_i) + 4\sqrt{\frac{2}{n - t - 1}\ln\frac{2(n - c_n + 1)}{\delta}} \geqslant K_t\right)\\
 &\qquad= \sum_{t = c_n - 1}^{n-2}\Pr\left(\nmlzi{1}\R(h_i) \geqslant M_{t,n} + \epsilon\right)\\
 &\qquad\leqslant (n-c_n - 1)\left[2\mathcal{N}\left(\mathcal{H},
\frac{\epsilon}{16\text{Lip}(\phi)}\right)\right]
\exp\left\{-\frac{(cn-1)\epsilon^2}{64} + \ln n\right\} \qquad (\text{By (\ref{eqn:gen1}).})\\
&\qquad\leqslant \left[2\mathcal{N}\left(\mathcal{H},
\frac{\epsilon}{16\text{Lip}(\phi)}\right)\right]
\exp\left\{-\frac{(cn-1)\epsilon^2}{64} + 2\ln n\right\}.
\end{split}
\]
Therefore, using (\ref{eqn:confbd}), we get
\[
\begin{split}
 &\Pr\left(\R(\widehat{h}) \geqslant \underset{c_n - 1 \leqslant t < n - 1}{\min}\ \left(M_{t, n} + 4\sqrt{\frac{2}{n - t - 1}\ln\frac{2(n - c_n + 1)}{\delta}}\right) + \epsilon\right)\\
 &\qquad \leqslant \delta + \left[2\mathcal{N}\left(\mathcal{H},
\frac{\epsilon}{16\text{Lip}(\phi)}\right)\right]
\exp\left\{-\frac{(cn-1)\epsilon^2}{64} + 2\ln n\right\},
\end{split}
\]
which, in particular, leads to
\[
\begin{split}
 &\Pr\left(\R(\widehat{h}) \geqslant  M^n + 4\sqrt{\frac{2}{n - c_n}\ln\frac{2(n - c_n + 1)}{\delta}} + \epsilon\right)\\
 &\qquad \leqslant \delta + \left[2\mathcal{N}\left(\mathcal{H},
\frac{\epsilon}{16\text{Lip}(\phi)}\right)\right]
\exp\left\{-\frac{(cn-1)\epsilon^2}{64} + 2\ln n\right\}.
\end{split}
\]
By substituting $\epsilon$ with $\epsilon/2$ and choosing $\delta$ as in the statement of Theorem~\ref{thm:concetra}, that is, satisfying
$
4\sqrt{\frac{2}{n - c_n}\ln\frac{2(n - c_n + 1)}{\delta}} = \frac{\epsilon}{2},$
we have for any $c>0$,
\[
\begin{split}
 &\Pr\left(\R(\widehat{h}) \geqslant  M^n + \epsilon\right)\leqslant 2(n-c_n + 1)\exp\left\{-\frac{(n-c_n)\epsilon^2}{64}\right\} \\
 &\qquad\quad + \left[2\mathcal{N}\left(\mathcal{H},
\frac{\epsilon}{32\text{Lip}(\phi)}\right) + 1\right]
\exp\left\{-\frac{(cn-1)\epsilon^2}{256} + 2\ln n\right\}\\
&\qquad \leqslant 2\left[\mathcal{N}\left(\mathcal{H},
\frac{\epsilon}{32\text{Lip}(\phi)}\right) + 1\right]
\exp\left\{-\frac{(cn-1)\epsilon^2}{256} + 2\ln n\right\}
\end{split}
\]
\end{proof}

\section{Proof of Theorem~\ref{thm:insepf}}
\label{sec:insepf}
\begin{proof}[\emph{Proof of Theorem~\ref{thm:insepf}}]
   First notice that $\w_0 = \w_1 = 0$ and from~(\ref{eqn:hingeineq}), we also have
  $
      \langle \uu, y_t(\x_t - \x_j)\rangle \geqslant \gamma - {\ell}_{j}^{t}
  $ whenever $y_t \neq y_j$ and ${\ell}_{j}^{t}=0$ otherwise. Thus similarly, we have

\begin{align}
\label{eqn:inseplb1}
  \langle \w_t, \uu \rangle &= \langle \w_{t-1}, \uu \rangle + \frac{1}{|\B_{t-1}|}\sum_{j\in\B_{t-1}}\ell^t_j\langle \uu, y_t(\x_t - \x_j)\rangle \nonumber\\
  & \geqslant \langle \w_{t-1}, \uu \rangle + \frac{1}{|\B_{t-1}|}\sum_{j\in\B_{t-1}}\ell_j^t (\gamma - \hat{\ell}_{j}^{t}(\uu))  \nonumber\\
  &= \langle \w_{t-1}, \uu \rangle + \frac{\gamma}{|\B_{t-1}|}\sum_{j\in\B_{t-1}}\ell_j^t - \frac{1}{|\B_{t-1}|}\sum_{j\in\B_{t-1}}\ell_j^t\cdot\hat{\ell}_{j}^{t}(\uu) \nonumber\\
  & \geqslant \langle \w_{t-1}, \uu \rangle + \frac{\gamma}{|\B_{t-1}|}\sum_{j\in\B_{t-1}}\ell_j^t - \frac{1}{|\B_{t-1}|}\sum_{j\in\B_{t-1}}\hat{\ell}_{j}^{t}(\uu) \qquad (\because \ell_j^t \in [0, 1]) \nonumber\\
  \Rightarrow\quad \langle \w_t, \uu \rangle &\geqslant \sum_{t=2}^n\left[\frac{\gamma}{|\B_{t-1}|}\sum_{j\in\B_{t-1}}\ell_j^t - \frac{1}{|\B_{t-1}|}\sum_{j\in\B_{t-1}}\hat{\ell}_{j}^{t}(\uu)\right] = \gamma M_{\B} - M^*_{\B}.
\end{align}

On the other hand, we have,
\begin{align}
\label{eqn:sepub1}
  \|\w_t\|^2 &= \|\w_{t-1}\|^2 + \frac{2}{|\B_{t-1}|}\sum_{j\in\B_{t-1}}\ell^t_j\langle \w_{t-1}, y_t(\x_t - \x_j)\rangle + \left\|\frac{1}{|\B_{t-1}|}\sum_{j\in\B_{t-1}}\ell^t_jy_t(\x_t - \x_j)\right\|^2\nonumber\\
  &\leqslant \|\w_{t-1}\|^2 +  \frac{2}{|\B_{t-1}|}\sum_{j\in\B_{t-1}}\ell^t_j + 4R^2\left(\frac{1}{|\B_{t-1}|}\right)^2\left(\sum_{j\in\B_{t-1}}\ell^t_j\right)\cdot \left(\sum_{j\in\B_{t-1}}\ell^t_j\right) \nonumber\\
  &\qquad (\because \ell_j^t > 0 \Rightarrow \langle \w_{t-1}, y_t(\x_t - \x_j)\rangle \leqslant 1)\nonumber\\
  &\leqslant \|\w_{t-1}\|^2 +  \frac{2}{|\B_{t-1}|}\sum_{j\in\B_{t-1}}\ell^t_j + 4R^2\left(\frac{1}{|\B_{t-1}|}\right)^2 \left(\sum_{j\in\B_{t-1}}\ell^t_j\right)\cdot |\B_{t-1}| \qquad (\because \ell_j^t \in [0, 1])\nonumber\\
  &= \|\w_{t-1}\|^2 + (4R^2 + 2)\left[\frac{1}{|\B_{t-1}|}\sum_{j\in\B_{t-1}}\ell^t_j\right]\nonumber\\
  \Rightarrow\quad \|\w_n\|^2 &\leqslant (4R^2 + 2)\sum_{t=2}^n\left[\frac{1}{|\B_{t-1}|}\sum_{j\in\B_{t-1}}\ell_j^t\right] = (4R^2 + 2)M_{\B}
\end{align}
Combining (\ref{eqn:inseplb1}) and (\ref{eqn:sepub1}), we have
$
(\gamma M_{\B} - M_{\B}^*)^2 \leqslant (4R^2 + 2)M_{\B},
$
which yields the desired bound.
%\[
%\begin{split}
%M_{\B} &\leqslant \frac{\gamma M_{\B}^* + (2R^2 + 1) + \sqrt{(2R^2+1)(\gamma M_{\B}^* + 2R^2 + 1)}}{\gamma^2} \\
%&\leqslant  \frac{\gamma M_{\B}^* + (4R^2 + 2) + \sqrt{(2R^2+1)\gamma M_{\B}^*}}{\gamma^2} \leqslant \left(\frac{\sqrt{4R^2 + 2} + \sqrt{\gamma M_{\B}^*}}{\gamma} \right)^2
%\end{split}
%\]
\end{proof}

\section{Proof of Theorem~\ref{thm:concetra2}}
\label{sec:appen:concetra2}
\begin{proof}[Proof of Theorem~\ref{thm:concetra2}]
Again, we rewrite our objective
\begin{equation}
\Pr_{Z^n\sim\D^n}\left(\nmlz\mathcal{R}(h_{t-1}) - \nmlz M^{\B}_t \geqslant \epsilon\right),
\end{equation}
as
\begin{align}
\label{eqn:probsplit1}
&\Pr\left(\nmlz\bigg[\mathcal{R}(h_{t-1}) - \bbbe_t[M^{\B}_t]\bigg] +  \nmlz\bigg[\bbbe_t[M^{\B}_t] - M^{\B}_t\bigg] \geqslant \epsilon\right) \nonumber\\
&\quad \leqslant
\Pr\left(\nmlz\bigg[\mathcal{R}(h_{t-1}) -
\bbbe_t[M^{\B}_t]\bigg] \geqslant
\frac{\epsilon}{2}\right) + \Pr\left(
\nmlz\bigg[\bbbe_t[M^{\B}_t] - M^{\B}_t\bigg]
\geqslant \frac{\epsilon}{2} \right).
\end{align}
Thus, we can bound the two terms separately. The proof consists of four parts, as follows.

\subsection*{Step 1: Bounding the Martingale difference}
First consider the second term in (\ref{eqn:probsplit1}). We have that $V_t =
(\bbbe_t[M^{\B}_t] - M^{\B}_t)/(n-c_n)$ is a martingale
difference sequence, i.e. $\bbbe_t[V_t] = 0$. Since the loss function is bounded in $[0,1]$, we have
$
|V_t| \leqslant {1}/(n - c_n),\ t = 1,\cdots, n.
$
Therefore by the Hoeffding-Azuma inequality, $\sum_t V_t$
can be bounded such that
\begin{equation}
\label{eqn:second2} \Pr_{Z^n\sim\D^n}\left( \nmlz\bigg[\bbbe_t[M^{\B}_t] - M^{\B}_t\bigg]
\geqslant \frac{\epsilon}{2} \right) \leqslant
\exp\left\{-\frac{(1-c)n\epsilon^2}{2}\right\}.
\end{equation}
\subsection*{Step 2: Symmetrization by a ghost sample $\Xi^n$}
Recall $M^{\B}_t$ and define $\widetilde{M}^{\B}_t$ as
\begin{equation}
\label{eqn:tm1} M^{\B}_t(Z^t) =
\nmlzb \ell\left(h_{t-1}, \z_t, \z_j\right), \qquad \widetilde{M}^{\B}_t(Z^t) = \nmlzb \ell\left(h_{t-1}, \z_t, \xi_j\right).
\end{equation}

\begin{claim}
\label{clm:sym1}
The following equation holds
{\small{\begin{equation}
\label{eqn:symm1}
\Pr_{Z^n\sim\D^n}\left(\nmlz\left[\mathcal{R}(h_{t-1}) - \bbbe_t[M^{\B}_t]\right]
\geqslant \epsilon\right) \leqslant 2\Pr_{\substack{Z^n\sim\D^n \\\Xi^n\sim\D^n}}\left(\nmlz
\left[\bbbe_t[\widetilde{M}^{\B}_t] - \bbbe_t[M^{\B}_t]\right]\geqslant
\frac{\epsilon}{2}\right),
\end{equation}}}
whenever $(1-c)^2n\geqslant 1/2$.
\end{claim}
Notice that the probability measure on the right hand side of (\ref{eqn:symm1}) is on $Z^n\times\Xi^n$.

\begin{proof}%[\emph{Sketch of the proof of Claim~\ref{clm:sym1}}]
It can be seen that the RHS (without the factor of 2) of (\ref{eqn:symm1}) is at least

\begin{displaymath}
\begin{split}
&\Pr_{\substack{Z^n\sim\D^n \\ \Xi^n\sim\D^n}}\left(\left\{\nmlz\left[\mathcal{R}(h_{t-1}) - \bbbe_t[M^{\B}_t]\right] \geqslant \epsilon\right\} \bigcap
\left\{\bigg|\nmlz\left[\bbbe_t[\widetilde{M}^{\B}_t] - \mathcal{R}(h_{t-1})\right]\bigg| \leqslant \frac{\epsilon}{2}\right\}\right)\\
&=
\bbbe_{Z^n\sim\D^n}\left[\bbbi_{\left\{\nmlz \left[\mathcal{R}(h_{t-1}) - \bbbe_t[M^{\B}_t]\right] \geqslant \epsilon\right\}}\cdot\Pr_{\Xi^n\sim\D^n}\left(\bigg|
\nmlz\left[\bbbe_t[\widetilde{M}^{\B}_t] - \mathcal{R}(h_{t-1})\right] \bigg|\leqslant \frac{\epsilon}{2}\bigg|Z^n\right)\right].
\end{split}
\end{displaymath}

Since $\bbbe_{\Xi^n\sim\D^n}\bbbe_t[\widetilde{M}^{\B}_t] = \mathcal{R}(h_{t-1})$, by Chebyshev's inequality
\begin{equation}
\label{eqn:cheyb1}
\begin{split}
\Pr_{\Xi^n\sim\D^n}\left( \bigg|\nmlz\bigg[\bbbe_t[\widetilde{M}^{\B}_t] - \mathcal{R}(h_{t-1})\bigg]\bigg| \leqslant \frac{\epsilon}{2}\bigg|Z^n\right) \geqslant 1 - \frac{\textbf{Var}\left\{\nmlz\bbbe_t[\widetilde{M}^{\B}_t]\right\}}{\epsilon^2/4}.
\end{split}
\end{equation}
To bound the variance, we first investigate the largest variation when changing one random variable $\bxi_j$ with others fixed.
From (\ref{eqn:tm1}), it can be easily seen that changing any of
the $\bxi_j$ varies each $\bbbe_t[\widetilde{M}^{\B}_t]$, where $t > j$ by at most by $1/|\B_t|$.

To estimate the difference the variation of $\nmlz\bbbe_t[\widetilde{M}^{\B}_t]$
regarding the $j$th example $\bxi_j$, it is easy to see that since we are using the FIFO strategy, any example can only stay in the buffer for $\B$ round. Changing each example can only result in $1/|\B|$ for $\bbbe_t[\widetilde{M}^{\B}_t]$. Thus, the variation is bounded by
\[
c_j =
\frac{1}{n-c_n}\bbbe_t[\widetilde{M}^{\B}_t]\leqslant \frac{1}{n-c_n}|\B|\frac{1}{|\B|} = \frac{1}{n-c_n}.
\]
Thus, we have
\begin{equation}
\label{eqn:sumci}
\begin{split}
\textbf{Var}\left(\nmlz\bbbe_t[\widetilde{M}^{\B}_t]\right) &\leqslant
\frac{1}{4}\sum_{i=1}^n c_i^2 = \frac{1}{4(1-c)^2n}.
\end{split}
\end{equation}
Thus, whenever $(1-c)^2n\geqslant 1/2$, the LHS of (\ref{eqn:cheyb1}) is greater or equal than $1/2$. This completes the proof of Claim~\ref{clm:sym1}.
\end{proof}
\subsection*{Step 3: Uniform Convergence}
In this step, we show how one can bound the RHS of (\ref{eqn:symm1}) using uniform convergence techniques, McDiarmid's inequality and $L_\infty$ covering number. Our task reduces to bound the following quantity
\begin{equation}
\label{eqn:final}
\Pr_{Z^n\sim\D^n, \Xi^n\sim\D^n}\left(\nmlz\left[\bbbe_t[\widetilde{M}^{\B}_t] - \bbbe_t[M^{\B}_t]\right]\geqslant {\epsilon}\right).
\end{equation}
Define
$
L_t(h_{t-1}) =\bbbe_t[\widetilde{M}^{\B}_t] - \bbbe_t[M^{\B}_t].
$
Thus we have
\begin{align}
\label{eqn:final21}
\Pr_{Z^n\sim\D^n, \Xi^n\sim\D^n}\left(\nmlz L_t(h_{t-1}) \geqslant \epsilon\right)
&\leqslant
\Pr\left(\underset{\hat{h}_{c_n},\cdots, \hat{h}_{n-1}}{\sup}\left[\nmlz L_t(\hat{h}_t)\right] \geqslant \epsilon\right)\nonumber\\
&\leqslant
\sum_{t=c_n}^{n-1}\Pr_{Z^t\sim\D^t, \Xi^t\sim\D^t}\left(\underset{\hat{h}\in\mathcal{H}}{\sup}\left[L_t(\hat{h})\right]\geqslant\epsilon\right).
\end{align}
To bound the RHS of (\ref{eqn:final21}), we start with the following lemma.% whose proof is given in the appendix.
\begin{lemma}
\label{lemma:lt2}
Given any function
$f\in\mathcal{H}$ and any $t\geqslant 2$
\begin{equation}
\label{eqn:Ltconc1}
\Pr_{Z^t\sim\D^t, \Xi^t\sim\D^t}\left(L_t(f)\geqslant\epsilon\right) \leqslant
\exp\left\{-{(|\B_t|-1)\epsilon^2}\right\}.
\end{equation}
\end{lemma}
 We have already proved the case when $t\leqslant |\B|$. When, $|\B| < t$, $L_t(f)$ has a bounded variation of $1/|\B_t|$ when changing its $2|\B_t|$ variables. Applying McDiarmid's inequality, we immediately get the desired inequality. Using exactly the same reasoning as in Lemma 6-8, we get
\begin{lemma}
\label{thm:supL1}
For every $2\leqslant t\leqslant n$, we have
\begin{equation}
\label{eqn:supLt1}
\Pr\left(\underset{h\in\mathcal{H}}{\sup}\left[L_t(h)\right]\geqslant\epsilon\right) \leqslant  \mathcal{N}\left(\mathcal{H},
\frac{\epsilon}{4\text{Lip}(\phi)}\right)\exp\left\{-\frac{(|\B_t|-1)\epsilon^2}{8}\right\}.
\end{equation}
\end{lemma}
Combining (\ref{eqn:supLt1}) and (\ref{eqn:final21}), we have
\begin{equation}
\label{eqn:ssupL2}
\Pr\left(\nmlz L_t(h_{t-1}) \geqslant \epsilon\right)
\leqslant \mathcal{N}\left(\mathcal{H},
\frac{\epsilon}{4\text{Lip}(\phi)}\right)
n\exp\left\{-\frac{(|\B_{cn}|-1)\epsilon^2}{8}\right\}.
\end{equation}

\subsection*{Step 4: Putting it all together}
We have
{\small{\begin{equation}
\label{eqn:first2}
\Pr_{Z^n\sim\D^n}\left(\nmlz\left(\mathcal{R}(h_{t-1}) -
\bbbe_t[M^{\B}_t]\right) \geqslant
\frac{\epsilon}{2}\right) \leqslant 2\mathcal{N}\left(\mathcal{H},
\frac{\epsilon}{16\text{Lip}(\phi)}\right)
n\exp\left\{-\frac{(|\B_{cn}| -1)\epsilon^2}{64}\right\}.
\end{equation}}}
\end{proof} 

\end{document}